\documentclass[11pt]{article}

\usepackage[final]{acl}

 \usepackage{microtype}

\usepackage[english,bidi=default]{babel} % English as the main language.
\babelfont{rm}{TeXGyreTermesX} % similar to Times
\babelprovide[import]{hindi}
\babelfont[*devanagari]{rm}{Lohit Devanagari}
\babelprovide[import]{arabic}
\babelfont[*arabic]{rm}{Noto Sans Arabic}

\usepackage{graphicx}
\usepackage{amsthm}
\usepackage{amsmath} 
\usepackage{amssymb}
\usepackage{booktabs}
\usepackage{multirow}
\usepackage{hyperref}
\usepackage{pifont}
\newtheorem{theorem}{Theorem}
\newtheorem{definition}{Definition}
\title{WISCA: A Lightweight Model Transition Method to Improve LLM Training via Weight Scaling}

\author{
    Jiacheng Li\textsuperscript{\rm 1},
    \ Jianchao Tan\textsuperscript{\rm 1} \thanks{\ \ Corresponding author.},
    \ Zhidong Yang\textsuperscript{\rm 3},
    \ Pingwei Sun\textsuperscript{\rm 1},
    \ Feiye Huo\textsuperscript{\rm 1}, 
    \ Jiayu Qin\textsuperscript{\rm 1} 
    \\
    \ {\bf Xiangyu Zhang}\textsuperscript{\rm 2},
    \ {\bf Maoxin He}\textsuperscript{\rm 4},
    \\
    \ {\bf Yerui Sun}\textsuperscript{\rm 1},
    \ {\bf Yuchen Xie}\textsuperscript{\rm 1},
    \ {\bf Guangming Tan}\textsuperscript{\rm 2},
    \ {\bf Weile Jia}\textsuperscript{\rm 2} ,
    \ {\bf Xunliang Cai}\textsuperscript{\rm 1},
    \ {\bf Tong Zhao}\textsuperscript{\rm 2} \footnotemark[1], \\
    \textsuperscript{\rm 1}Meituan, Beijing, China \\
    \textsuperscript{\rm 2}University of Chinese Academy of Sciences, Beijing, China \\
    \textsuperscript{\rm 3}Hong Kong University of Science and Technology, Hong Kong SAR, China \\
    \textsuperscript{\rm 4}Xiamen University, Xiamen, China \\
    \texttt{\{lijiacheng14, tanjianchao02\}@meituan.com},
    \texttt{\{jiaweile, tongzhao.zhaot\}@gmail.com}
}
\begin{document}

\maketitle
\begin{abstract}
% This document provides an example showing how
% to use the *ACL style files with either
% LuaLaTeX or XeLaTeX.
Transformer architecture gradually dominates the LLM field. Recent advances in training optimization for Transformer-based large language models (LLMs) primarily focus on architectural modifications or optimizer adjustments. However, these approaches lack systematic optimization of weight patterns during training. Weight pattern refers to the distribution and relative magnitudes of weight parameters in a neural network. To address this issue, we propose a Weight Scaling method called WISCA to enhance training efficiency and model quality by strategically improving neural network weight patterns—without changing network structures. By rescaling weights while preserving model outputs, WISCA indirectly optimizes the model’s training trajectory. Experiments demonstrate that WISCA significantly improves convergence quality (measured by generalization capability and loss reduction), particularly in LLMs with Grouped Query Attention (GQA) architectures and LoRA fine-tuning tasks. Empirical results show \textbf{5.6\%} average improvement on zero-shot validation tasks and \textbf{2.12\%} average reduction in training perplexity across multiple architectures.
\end{abstract}

\section{Introduction}

% Please see the general instructions
% in the file \verb|acl_latex.tex|.

% Here are some examples of text in various languages.

% Hindi: \foreignlanguage{hindi}{मानव अधिकारों की सार्वभौम घोषणा}

% Arabic: \foreignlanguage{arabic}{الإعلان العالمي لحقوق الإنسان}

% Here is an example citation:
% \citet{Gusfield:97} argues that...

% 神经网络不同权重分布会导致不同的收敛难度。例如理论上两层全连接神经网络可以拟合任意函数，然而其在复杂任务上的效果往往很糟糕。而针对不同任务涌现的多种复杂的神经网络架构如CNN，GNN，RNN，Transformer等架构可以看作针对特定任务的与简单全连接任务相比不同权重分布与设计。所以这些巧妙设计的神经网络架构本质上是比全连接神经网络更合理的权重分布设计。
The weight pattern is defined as the arrangement and scaling of parameters across layers in a neural network, which plays a pivotal role in the convergence of neural networks. For example, although a two-layer fully connected network theoretically can approximate any function (universal approximation theorem)~\citet{256500}, its performance on complex tasks is often suboptimal due to poorly structured weight patterns. In contrast, modern architectures (e.g., CNNs, GNNs, RNNs, Transformers) achieve superior results by implicitly optimizing weight patterns for specific tasks. These architectures reduce optimization complexity and enhance representational capability through structured designs (e.g., local receptive fields in CNNs~\citet{lecun2002gradient}, attention mechanisms in Transformers~\citet{vaswani2017attention}), demonstrating that an effective weight pattern is a key factor in their success.

\begin{figure}[t]
\centering
\includegraphics[width=\columnwidth]{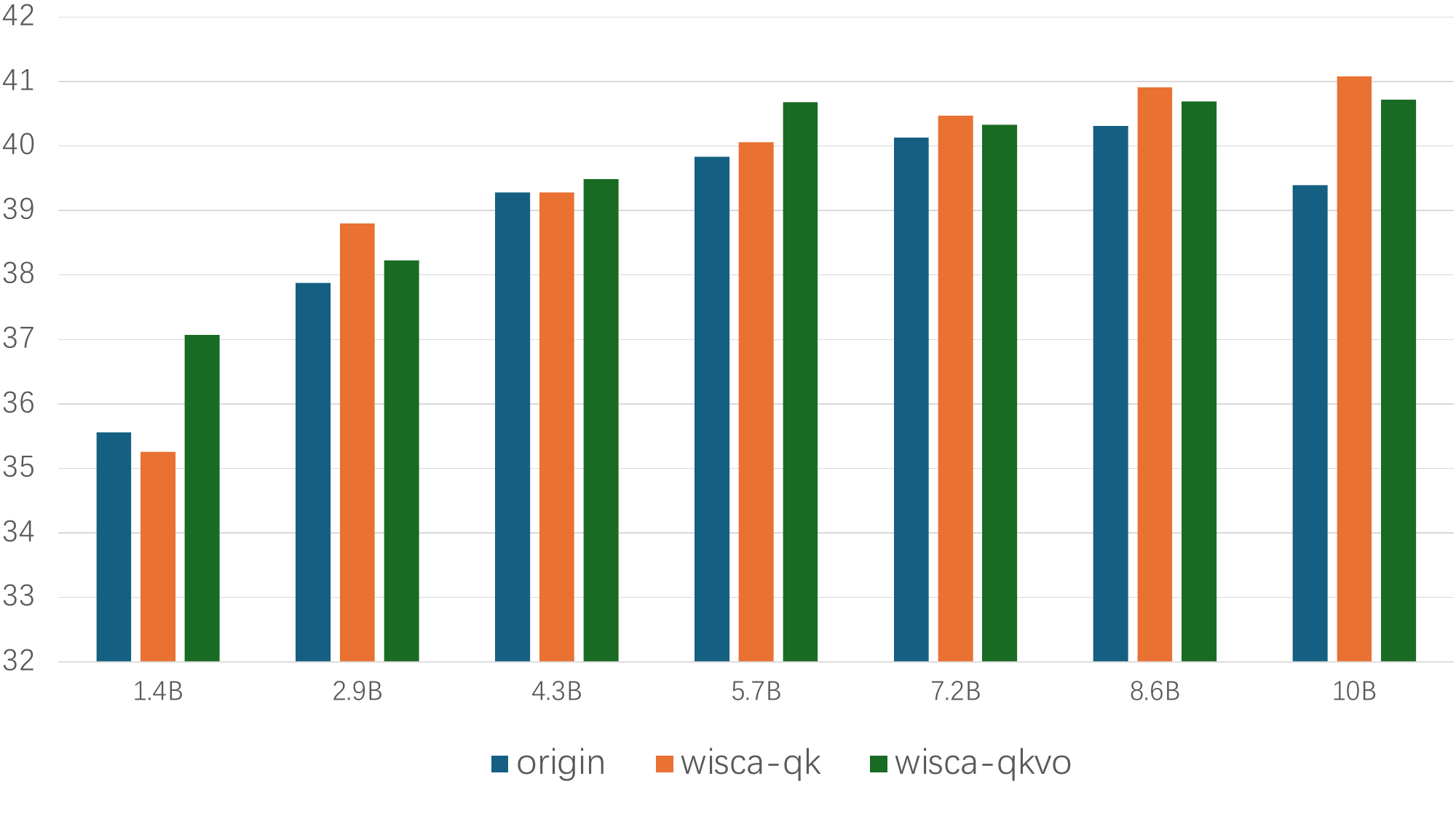}
\caption{The comparisons of average zero-shot evaluation results at different training steps using WISCA and original methods on llama-moe-5B-A0.8B. All metrics can be seen in Appendix.}
\label{fig:wisca-bar}
\vskip -0.2in
\end{figure}

% 一种典型的糟糕的神经网络权重分布导致糟糕的模型效果的例子是尖锐极小值问题。在[xxx]的实验中表明，如果模型使用较大的batch-size训练，会更加容易落入尖锐极小值点。而处于尖锐极小值点的模型相比于平坦极小值点的模型在测试集上的泛化性能更差。这样的结果是训练集与测试集之间的bias带来的：由于训练集与测试集之间的差异，不同模型在他们之间计算得到的loss不同，并且由于训练集与测试集之间的密切关系，其在相近的模型上的效果又有所近似，具体表现为在landscape曲线上趋势一致但不重合。这里我们可以看到尖锐级小镇对模型泛化性的负作用。
A typical example of a poor weight pattern in a neural network that leads to suboptimal model performance is the sharp minimum problem. As demonstrated in \citet{keskar2016large}, training with excessively large batch sizes increases the likelihood of converging to sharp minima. Models trapped in sharp minima exhibit significantly worse generalization on test datasets compared to those in flat minima. This discrepancy arises from the bias between the training and test sets: Firstly, the loss landscapes computed on training and test datasets diverge significantly due to their distributional differences. Secondly, because of the distributional relationship between the two datasets being close, the losses of them exhibit relatively similar but inconsistent  trends in parameter spaces. The loss landscape curves of them are parallel but non-overlapping trajectories. Thus, the side impact of sharp minima on generalization is apparent. It amplifies sensitivity to outliers in the dataset and weakens the model's robustness as a result.

% 在非凸优化中，崎岖的收敛路径——以损失函数景观中的尖锐极小值和高度弯曲区域为特征——对基于梯度的优化方法构成重大挑战。此类路径迫使优化器在狭窄的沟壑和鞍点间穿行，往往导致模型过早收敛至泛化性能较差的次优解（如尖锐极小值）。相比之下，平滑且平坦的收敛路径允许梯度将参数引导至更宽广的极小值区域，此时权重的小幅扰动对损失的影响微乎其微。这种稳定性本质上提升了泛化能力，因为处于平坦区域的模型对数据集特异性噪声的敏感性更低。这种差异凸显了损失函数景观几何结构对优化成功与否的关键作用。
In non-convex optimization, rugged convergence paths, characterized by sharp minima and highly curved regions in the loss landscape, pose significant challenges to gradient-based optimization. Such paths force the optimizer to wander through narrow valleys and saddle points, leading to premature convergence to suboptimal solutions with poor generalization (e.g., sharp minima). In contrast, a smooth, flat convergence path allows the gradients to guide parameters in approaching a wider minimum, where small perturbations in the weights have minimal side-impact on the loss. This stability inherently improves generalization, as models in flat regions are less sensitive to outliers in the dataset. This difference highlights the critical role of loss landscape geometry in successfully converging to a wider minimum.

% 我们从优化神经网络训练路径与收敛结果的权重分布的角度出发，提出了等价模型理论概念和面向神经网络权重分布的优化策略WISCA。等价模型概念指：若两个神经网络模型结构相同，并且针对同样的输入得出相同的结果，但是模型之间的权重不同，这样的模型互为等价模型。模型训练过程中可以切换到中间结果的等价模型进行继续训练，而WISCA指出了一种切换到更合适的权重分布的策略。
From the perspective of optimizing weight patterns to improve neural network training trajectories in the loss landscape and the results of convergence, we propose the Equivalent Model Theory and an optimization strategy for weight patterns (WISCA). Equivalent Model Theory gives the definition of equivalent models. Consider two models; if they have the same architecture and derive the same output with the same input despite the differences in weight configurations of them. Such two models can be regarded as equivalent models. During training, models can transition to equivalent models, and the training process will continue. WISCA introduces a systematic strategy to guide such transitions to superior weight patterns, thereby indirectly reshaping the training process dynamically. By prioritizing weight configurations that align with smoother optimization landscapes, WISCA enhances the likelihood of converging to a flat minimum with superior generalization performance.

% 具体而言，我们的贡献如下：
Our contributions can be summarized as follows:

\begin{itemize}

% 提出等价模型理论，为通过权重分布优化提升模型训练效果奠定理论基础
\item We propose the Equivalent Model Theory, which establishes a theoretical foundation for enhancing model training performance through weight pattern optimization. 

% 我们提出了一种针对特殊的等价模型转换策略WISCA，能够使transformer架构的模型训练中权重动态转换到更合理的分布，并与转换前的模型互为等价模型。
\item Based on the proposed theory, we propose a novel model transition strategy called WISCA, which enables Transformer-based models and LoRA architectures to dynamically adjust their weight patterns to more optimal configurations during training while remaining equivalent to the unadjusted model.

% 本研究开创性地实现了优化策略的范式转变，将关注点从架构设计转向动态权重调整机制。
\item Extensive experimental results show that our WISCA can improve model performance through dynamically adjusting weights by finding equivalent models.
\end{itemize}

\section{Theory}

% 尖锐级小值带来更差的泛化性
\subsection{Sharp Minima leads to Worse Generalization}

\begin{theorem}[Generalization Gap of Sharp vs. Flat Minima]
\label{thm:sharp_flat}
Let $\theta_1$ and $\theta_2$ be two models with identical training loss $L_{\text{train}}(\theta_1) = L_{\text{train}}(\theta_2) = L_0$. If $\theta_1$ resides in a sharp minimum (defined as $\exists~\mathcal{N}(\epsilon), \text{for}~\forall \epsilon \in \mathcal{N}(\epsilon), \text{where}~\|\epsilon\| \to 0,~L(\theta_1 + \epsilon) - L(\theta_1) > L(\theta_2 + \epsilon) - L(\theta_2) $), then the expected validation loss satisfies:
\begin{equation}
    \mathbb{E}[L_{\text{val}}(\theta_1)] > \mathbb{E}[L_{\text{val}}(\theta_2)]
\end{equation}
\end{theorem}

\begin{proof}
The proof proceeds in five steps:

\textbf{Taylor Expansion of Loss Landscape:}  
For a small perturbation $\epsilon \in \mathbb{R}^d$, the loss at $\theta + \epsilon$ is approximated by:
\begin{equation}
    L(\theta + \epsilon) \approx L(\theta) + \frac{1}{2} \epsilon^\top H(\theta) \epsilon
\end{equation}
where $H(\theta) = \nabla^2 L(\theta)$ is the Hessian matrix. At minima ($\nabla L(\theta) = 0$), the first-order term vanishes.

\textbf{Sharpness Characterization:}  
The sharpness condition implies that for all $\epsilon$:
\begin{equation}
    \epsilon^\top H(\theta_1) \epsilon > \epsilon^\top H(\theta_2) \epsilon \quad \Rightarrow \quad H(\theta_1) \succeq H(\theta_2)
\end{equation}
where $\succeq$ denotes the positive semi-definite ordering.

\textbf{Modeling Validation Loss:}  
Assume validation data introduces a Gaussian perturbation $\delta \sim \mathcal{N}(0, \sigma^2 I)$ to the parameters. The validation loss is approximated as:
\begin{equation}
    L_{\text{val}}(\theta) \approx L_0 + \frac{1}{2} \delta^\top H(\theta) \delta
\end{equation}

\textbf{Expectation Calculation:}  
Taking expectation over $\delta$:
\begin{equation}
    \mathbb{E}[L_{\text{val}}(\theta)] = L_0 + \frac{\sigma^2}{2} \text{Tr}(H(\theta))
\end{equation}
where $\text{Tr}(H(\theta))$ is the trace of the Hessian.

\textbf{Inequality Derivation:}  
Since $H(\theta_1) \succeq H(\theta_2)$, it follows that $\text{Tr}(H(\theta_1)) \geq \text{Tr}(H(\theta_2))$. Thus:
\begin{align}
    &\mathbb{E}[L_{\text{val}}(\theta_1)] - \mathbb{E}[L_{\text{val}}(\theta_2)]\\ 
    &= \frac{\sigma^2}{2} \left( \text{Tr}(H(\theta_1)) - \text{Tr}(H(\theta_2)) \right) 
    \nonumber > 0
\end{align}
\end{proof}

% 等价模型
\subsection{Equivalent Model}

% 给出定义：若两个神经网络模型结构相同，并且针对同样的输入得出相同的结果，但是模型之间的权重不同，这样的模型互为等价模型。
\begin{definition}[Equivalent Models]
\label{def:equivalent_models}
Let $\mathcal{F}$ denote a neural network architecture subject to parameter space $\Theta$, input space $\mathcal{X}$, and output space $\mathcal{Y}$. Two parameter configurations $\theta_1, \theta_2 \in \Theta$ are called \textbf{equivalent models} if they satisfy:
\begin{enumerate}
    \item \textbf{Architectural Consistency}: $\theta_1$ and $\theta_2$ belong to the same architecture $\mathcal{F}$.
    \item \textbf{Functional Equivalence}: For all inputs $\mathbf{x} \in \mathcal{X}$,
    \begin{equation}
        F(\mathbf{x}; \theta_1) = F(\mathbf{x}; \theta_2)
    \end{equation}
    where $F(\cdot; \theta)$ denotes the function of forward propagation for architecture $\mathcal{F}$ with parameters $\theta$.
    \item \textbf{Parameter Distinction}: $\theta_1 \neq \theta_2$ (i.e., their weight patterns differ).
\end{enumerate}
\end{definition}

Equivalent models represent distinct points in the parameter space $\Theta$ that map to the same functional behavior in the output space $\mathcal{Y}$. This concept enables optimization strategies (e.g., our proposed WISCA) to transition between equivalent models during training, effectively reshaping weight patterns without altering functional outputs.

% 以全连接网络中的权重置换为例，当两层全连接网络的第l层的激活函数为relu，并且bias=0时，【Wl, w(l+1)】与【\alpha Wl, \frac{1}{\alpha} w(l+1)】两个模型具有相同的模型架构但权重分布不同，对于任意输入的结果恒定相等，这两个模型就互为等价模型。
For example, consider a two-layer fully connected network where:  

1. Layer \( l \) uses \(\mathrm{ReLU}(\cdot)\) activation with zero bias, 

2. Parameters are \((\mathbf{w}^{(l)}, \mathbf{w}^{(l+1)})\) and \((\alpha \mathbf{w}^{(l)}, \alpha^{-1} \mathbf{w}^{(l+1)})\) for \(\alpha > 0\).  

For any input \(\mathbf{x}\), the outputs are:  
\begin{align}
    \theta_1: & \quad \mathbf{w}^{(l+1)} \mathrm{ReLU}\left(\mathbf{w}^{(l)} \mathbf{x}\right) \\
    \theta_2: & \quad \alpha^{-1} \mathbf{w}^{(l+1)} \mathrm{ReLU}\left(\alpha \mathbf{w}^{(l)} \mathbf{x}\right)
\end{align}  
Since \(\mathrm{ReLU}(\alpha \mathbf{z}) = \alpha \mathrm{ReLU}(\mathbf{z})\) for \(\alpha > 0\), \(\theta_2\) simplifies to:  
\[
    \alpha^{-1} \mathbf{w}^{(l+1)} \cdot \alpha \mathrm{ReLU}\left(\mathbf{w}^{(l)} \mathbf{x}\right) = \mathbf{w}^{(l+1)} \mathrm{ReLU}\left(\mathbf{w}^{(l)} \mathbf{x}\right)
\]  
which matches \(\theta_1\) exactly.  

Thus, \(\theta_1\) and \(\theta_2\) are \textbf{equivalent models} (Definition~\ref{def:equivalent_models}):  
\begin{itemize}
    \item Identical architecture and outputs for all inputs
    \item Distinct weight patterns (\(\mathbf{w}^{(l)} \neq \alpha \mathbf{w}^{(l)}\))
\end{itemize}

% 一个理想的通过等价模型策略优化神经网络训练路径的方案是，在神经网络使用一阶优化器收敛点局部最优解时难以逃离时，在全范围搜索收敛到的局部最优解模型\theta_0的等价模型列表[\theta_1, \theta_2, ..., \theta_n]，并在[\theta_1, \theta_2, ..., \theta_n]中选择一个更容易继续收敛下去的模型\theta_k，然后直接从\theta_0的继续训练转换到\theta_k的继续训练。
An ideal implementation of the equivalent model strategy for optimizing neural network training trajectories consists of the following steps:

% 在通过等价模型转换策略进行模型收敛的优化中，理想情况下可以达到全局最优解。\theta_{init}表示神经网络初始化的模型，\theta_0表示若只使用传统的优化策略收敛到的最终检查点，\theta_1～\theta_4表示在模型全局范围内搜索到的所有等价模型，\theta_{final}表示全局最优解。在找到所有等价模型后，通过对所有等价模型[\theta_1,...,\theta_n]的继续收敛性评估，选择最优的等价模型\theta_1进行继续训练。若能够找到一个神经网络结构的所有等价模型，理论上可以收敛到全局最优解。

\begin{figure}[t]
\centering
\includegraphics[width=\columnwidth]{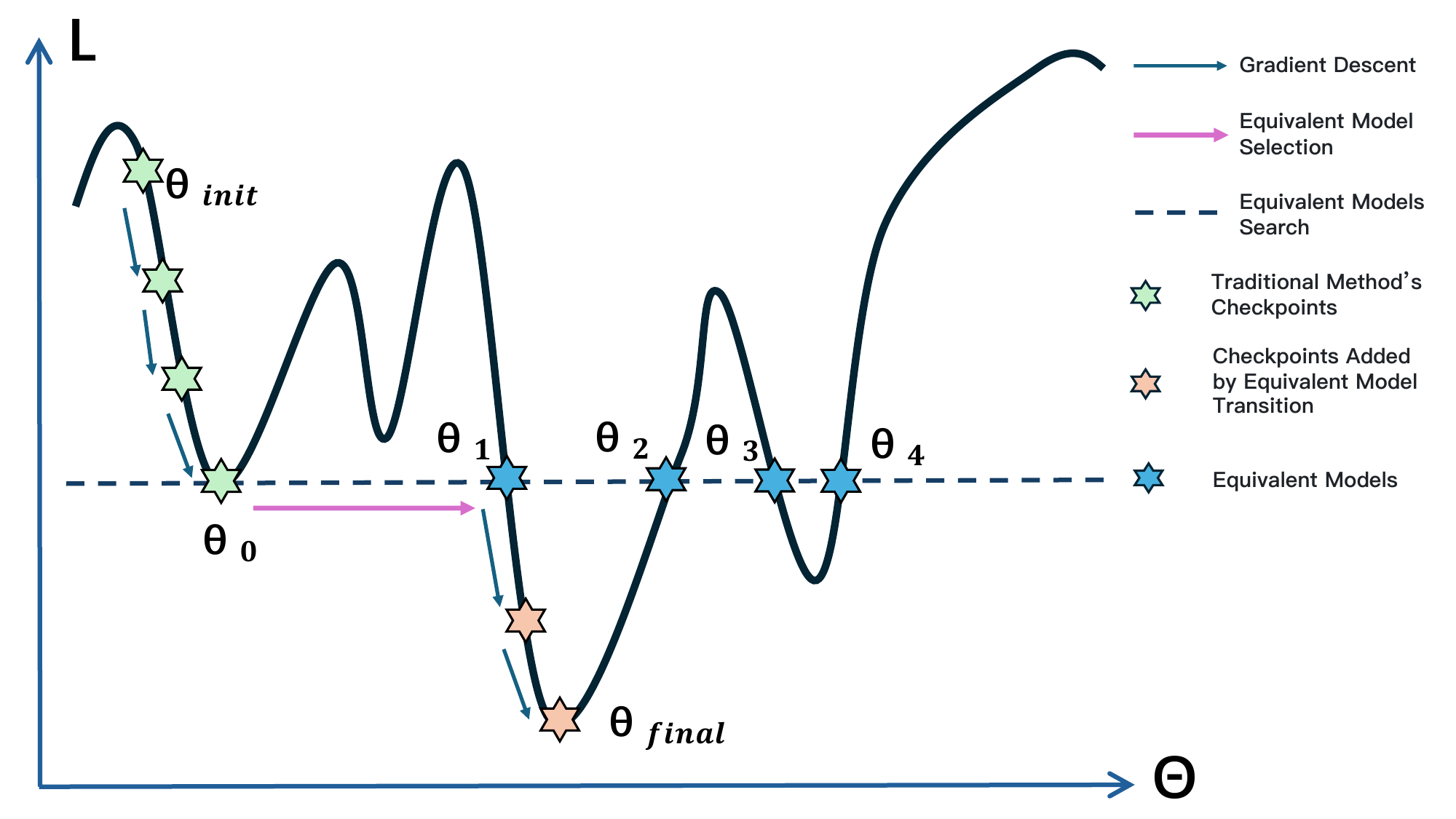}
\caption{\textbf{Optimization via Equivalent Model Transition.} 
In the ideal case, the global optimum $\theta_{\text{final}}$ can be achieved through this strategy. 
Key notations: 
$\theta_{\text{init}}$ (initialized model), 
$\theta_0$ (final checkpoint under conventional optimization), 
$\theta_1, \dots, \theta_n$ (all equivalent models explored globally), 
$\theta_{\text{final}}$ (global optimum). 
After identifying equivalent models, the optimal candidate (e.g., $\theta_1$) is selected for continued training based on convergence potential. 
Theoretically, convergence to $\theta_{\text{final}}$ is guaranteed if all equivalent models are evaluable.}
\label{fig:placeholder}
\vskip -0.15in
\end{figure}

% \begin{figure}[ht]
% \begin{center}
% \centerline{\includegraphics[width=\columnwidth]{./pictures/ideal_eqt_optim.pdf}} % 替换为空白占位符
% \caption{
% \textbf{Optimization via Equivalent Model Transition.} 
% In the ideal case, the global optimum $\theta_{\text{final}}$ can be achieved through this strategy. 
% Key notations: 
% $\theta_{\text{init}}$ (initialized model), 
% $\theta_0$ (final checkpoint under conventional optimization), 
% $\theta_1, \dots, \theta_n$ (all equivalent models explored globally), 
% $\theta_{\text{final}}$ (global optimum). 
% After identifying equivalent models, the optimal candidate (e.g., $\theta_1$) is selected for continued training based on convergence potential. 
% Theoretically, convergence to $\theta_{\text{final}}$ is guaranteed if all equivalent models are evaluable.
% }
% \label{fig:placeholder}
% \end{center}
% % \vskip -0.5in
% \end{figure}

\textbf{Escape from Local Minima:} When a first-order optimizer (e.g., SGD, Adam) becomes trapped in a sharp local minimum $\theta_0$ with poor generalization, systematically explore the set of equivalent models $[\theta_1, \theta_2, ..., \theta_n]$ (as defined in Definition \ref{def:equivalent_models}).

\textbf{Model Selection:} From the candidate set $[\theta_1, \theta_2, ..., \theta_n]$, select a model $\theta_k$ that lies in a "flatter" region of the loss landscape, since such regions are empirically linked to better generalization.

\textbf{Training Transition:} Resume training from $\theta_k$ instead of $\theta_0$, leveraging the functional equivalence of the models to maintain task performance while improving optimization dynamics.

% 理论上若能找到某神经网络结构的所有等价模型策略，则模型能够收敛至全局最优解，然而这几乎是不可能实现的。我们提出了一种寻找等价模型的策略WISCA，一种在全局范围内查找到部分等价模型后的继续训练策略，在实验中表现出了出色的成绩。

% \textbf{W}eighted \textbf{I}terative \textbf{S}earch for \textbf{C}onvergent \textbf{A}lternatives (\textbf{WISCA})
Theoretically, if one can identify all equivalent model strategies for a neural network architecture, the model can converge to the global optimum. However, it is quite challenging to implement in practice. We propose a strategy named WISCA, model transition during training that identifies partial equivalent models, which have demonstrated outstanding performance in extensive experiments.

% WISCA
\section{WISCA}
\label{def:wisca}

\label{sec:wisca}
% 在常见的transformer结构中，自注意力部分的表达关系为：X通过w_q与w_k后得到Q与K，Q与K的转置矩阵乘后与V矩阵乘，得到的结果经过w_o层得到输出。
For an input sequence $\mathbf{X} \in \mathbb{R}^{n \times d_{\text{model}}}$ with $n$ tokens and $d_{\text{model}}$-dimensional embeddings, the self-attention operation is formulated as follows:

\begin{align}
    \left( \mathbf{Q}, \mathbf{K}, \mathbf{V} \right) &= \mathbf{X} \left( \mathbf{W}_q, \mathbf{W}_k, \mathbf{W}_v \right)
\end{align} 
\begin{align}
    \text{Att}(\mathbf{Q}, \mathbf{K}, \mathbf{V}) &= \text{softmax}\left( \frac{\mathbf{Q}\mathbf{K}^\top}{\sqrt{d_k}} \right)\mathbf{V}
\end{align} 
\begin{align}
    \text{Output} &= \text{Att}(\mathbf{Q}, \mathbf{K}, \mathbf{V}) \mathbf{W}_o
\end{align}

% \begin{align}
%     % Linear projections
%     \mathbf{Q} &= \mathbf{X} \mathbf{W}_q, \quad \mathbf{W}_q \in \mathbb{R}^{d_{\text{model}} \times d_k} \tag{Query Projection} \\
%     \mathbf{K} &= \mathbf{X} \mathbf{W}_k, \quad \mathbf{W}_k \in \mathbb{R}^{d_{\text{model}} \times d_k} \tag{Key Projection} \\
%     \mathbf{V} &= \mathbf{X} \mathbf{W}_v, \quad \mathbf{W}_v \in \mathbb{R}^{d_{\text{model}} \times d_v} \tag{Value Projection}
% \end{align}

% % Scaled dot-product attention
% \begin{equation}
%     \text{Attention}(\mathbf{Q}, \mathbf{K}, \mathbf{V}) = \text{softmax}\left( \frac{\mathbf{Q}\mathbf{K}^\top}{\sqrt{d_k}} \right) \mathbf{V} \in \mathbb{R}^{n \times d_v}
% \end{equation}

% % Output projection
% \begin{equation}
%     \text{Output} = \text{Attention}(\mathbf{Q}, \mathbf{K}, \mathbf{V}) \mathbf{W}_o, \quad \mathbf{W}_o \in \mathbb{R}^{d_v \times d_{\text{model}}}
% \end{equation}

% 我们可以将W_q与W_k组成的参数组看成transformer结构中的一个组件，即\theta={w_q,w_k}。\theta组件的landscape图可以用mse来模拟绘制：L=(QK-C)^2，C为\theta组件的理想损失值，是一个常数。如果我们可以实现所有等价模型策略，那么就可以在landscape图（Loss-Q-K）中的等高线上做任意跳转。绘制出L=(QK-1)^2的三维曲线，可以在俯视图等高线上观察到，当Q=K时，不同等高线之间距离越远，也就代表着在QK=C的等高线上的所有等价模型中，Q=K时的landscape最平坦。

For the \(\theta\)-component (\(\theta = \{W_q, W_k\}\)), we model its loss landscape as:  
\begin{equation}
    \mathcal{L}(\mathbf{Q}, \mathbf{K}) = (\mathbf{Q}\mathbf{K} - C)^2, \quad C \in \mathbb{R}^+.
\end{equation}  
where, \(C\) represents the ideal value of \(\mathbf{Q}\mathbf{K}\) (e.g., \(C = 1\) for normalized attention).

Under the assumption of perfect equivalent model strategies, parameters can traverse \emph{any point} on the contour lines of \(\mathcal{L}(\mathbf{Q}, \mathbf{K}) = \text{const}\) (Figure~\ref{fig:wisca-sgd}). On the contour lines of the loss landscape, the region where $\mathbf{Q}=\mathbf{K}$ exhibits the flattest geometry. This is reflected in the top-down contour lines, where the distance between adjacent contour lines reaches the maximum when $\mathbf{Q}=\mathbf{K}$.

% 我们认为更加平滑的曲面是比同loss下更加尖锐崎岖的曲面质量更高的，这不仅仅表现在泛化性上，同时也表现在继续训练的收敛效果上。WISCA的核心思想为改变transformer的Q值和K值，使得在QK矩阵乘不变的情况下尽量使得Q=K。Figure~\ref{fig:wisca-sgd}比较了在随机初始化上的与随机初始化点经过WISCA调整为Q=K时的不同随机梯度下降效果。可以观察到使用WISCA后的收敛路径也会更加平坦与稳定。

We assume that flatter regions of the loss landscape are of higher quality than sharper, more rugged regions at the same loss level. This superiority manifests not only in generalization performance but also in the efficiency of continued training. The core idea of WISCA is to adjust the $\mathbf{Q}$ and $\mathbf{K}$ values by scaling in Transformer-based architectures to satisfy $\mathbf{Q}=\mathbf{K}$ while preserving the $\mathbf{QK}^T$ product.

As illustrated in Figure~\ref{fig:wisca-sgd}, WISCA modifies the initialization by enforcing $\mathbf{Q}=\mathbf{K}$, leading to fundamentally different optimization dynamics compared to random initialization. SGD-M-WISCA exhibits smoother and more stable convergence, avoiding sharp minima with WISCA-adjusted initialization.

% 不同的初始化点在曲面L=(QK-1)^2上的随机梯度下降法对比图显示，在使用WISCA算法的初始化点的优化路径更光滑平坦，并且迭代次数更少。而在与WISCA初始化点的同一loss的等价初始化模型上的优化路径则更加摇摆震荡，最终收敛经过的迭代次数更多。WISCA初始化点在第7次迭代收敛，而随机初始化点的优化路径在第25次迭代收敛。loss停止阈值为1e-2，SGD学习率0.01，动量系数0.9。在同一等高线上的所有点中，Q=K时更加平坦，所以使用WISCA可以使模型更容易收敛至光滑局部最小值点，并且加速模型收敛速度。
\begin{figure}[ht]
\begin{center}
\centerline{\includegraphics[width=\columnwidth]{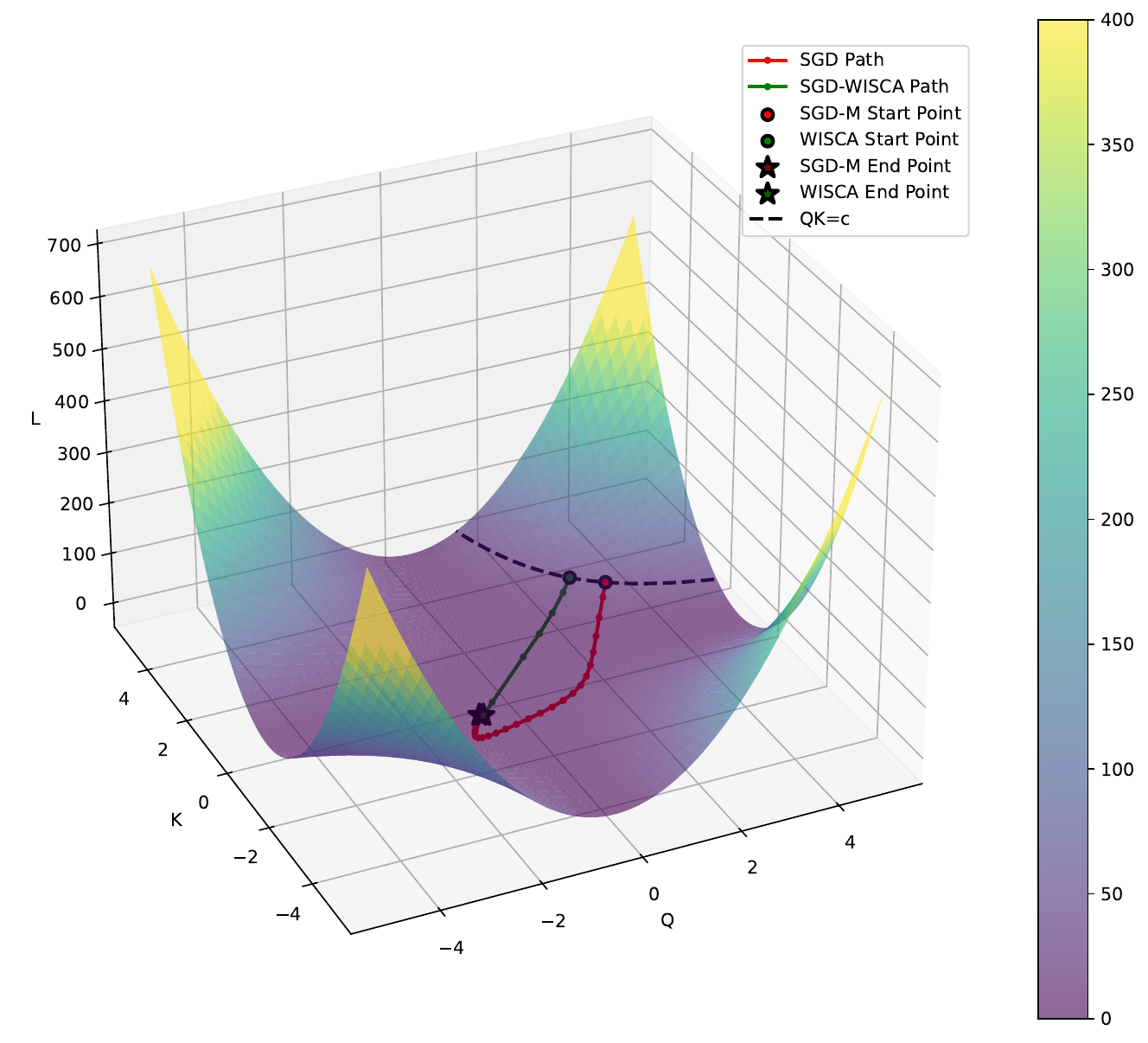}}
\caption{
\textbf{Comparison of SGD-M~\citet{bottou2007tradeoffs} and SGD-M-WISCA Optimization Paths.} 
\textbf{Red}: Optimization trajectory of SGD with momentum (\(\beta=0.9\)) starting from a random initialization. \textbf{Green}: Trajectory using WISCA-adjusted initialization (\(Q=K\)). Both methods minimize \(L = (QK - 1)^2\) with learning rate \(\eta=0.01\) and loss threshold \(\epsilon=10^{-2}\). 
\textbf{Observations}:  
1. The WISCA path (\textbf{Green}) converges in \textbf{7 iterations} with a smooth, flat trajectory, while the random initialization (\textbf{Red}) requires \textbf{25 iterations} with oscillatory behavior.  
2. All equivalent models on the same contour line (\(QK=C\)) exhibit slower convergence except at \(Q=K\), where the flattest region enables faster stabilization.  
This demonstrates WISCA’s ability to identify high-quality initializations aligned with flat minima, accelerating convergence without architectural changes.
}
\label{fig:wisca-sgd}
\end{center}
\vskip -0.4in
\end{figure}

% 除了在初始化阶段应用WISCA，也可以在训练过程中的任意iteration后进行WISCA调整后继续训练，因为WISCA调整前后互为等价模型。
WISCA can be applied not only during initialization but also at any iteration of the training process. Since the WISCA-adjusted model remains functionally equivalent to the original model (Definition~\ref{def:equivalent_models}), training can seamlessly resume from the transformed parameters.

% 在transformer块中，可以调整Wq和Wk，使得Wq和Wk的L1范数相等来保证计算得到的Q与K的相似，但需要保调整前后Q与K^T的矩阵乘结果不变。我们设计了下面公式：

Next, we will introduce the proposed WISCA in detail. The design of WISCA starts with vanilla self-attention in Transformer. To ensure similarity between $\mathbf{Q}$ and $\mathbf{K}$ while preserving the $\mathbf{QK}^T$ product, we devise the following adjustment for $W_q$ and $W_k$ in WISCA:  
\begin{align}
    W_q' &= W_q \cdot \sqrt{\frac{\|W_k\|_1}{\|W_q\|_1}} \label{eq:wq_adjust} \\
    W_k' &= W_k \cdot \sqrt{\frac{\|W_q\|_1}{\|W_k\|_1}} \label{eq:wk_adjust}
\end{align}  
where \( \| \cdot \|_1 \) denotes the \( L_1 \)-norm of a matrix (sum of absolute values of its elements).

Why does the convergence become stable after using WISCA? Here is an intuitive explanation. Considering the landscape $f(Q,K)=(QK-1)^2$ and the current checkpoint is $(Q,K)$, the gradient is $2(QK-1)(K,Q)$. After updating the parameter along the negative gradient direction with learning rate $\eta=\frac{\epsilon}{2(QK-1)}$, the next checkpoint is $(Q-\epsilon K, K-\epsilon Q)$. Now the gradient at this checkpoint is $2((Q-\epsilon K)(K-\epsilon Q)-1)(K-\epsilon Q, Q-\epsilon K)$. We hope the gradient direction varies as small as possible. Therefore, we have 
$
\frac{Q}{K}=\frac{Q-\epsilon K}{K-\epsilon Q}
$
and obtain $K^2=Q^2$, which means $|Q|=|K|$.

The proposed scaling strategy is also valid for consecutive linear layers, even when activation functions (e.g., ReLU, LeakyReLU~\citet{xu2015empirical}) are used. Similarly, it is also valid for $w_v$ and $w_o$ in transformer.

% wv和wo也同样适用于此weight-scaling策略，当wv的激活函数为relu或lekyrelu时，同样需要保证wv与wo的范数相同，证明方式与wq与wk相同只不过wq,wk模块是用（x wq）(x wk)模拟的，而wv和wo是以(x wv)wo模拟的，可以得到相同的结论。
To ensure \( \|W_v'\|_1 = \|W_o'\|_1 \) while preserving the output \( \text{Output} = (attention\_score\cdot V)\cdot W_o \), we define:  
\begin{align}
    W_v' &= W_v \cdot \sqrt{\frac{\|W_o\|_1}{\|W_v\|_1}} \label{eq:wv_adjust} \\
    W_o' &= W_o \cdot \sqrt{\frac{\|W_v\|_1}{\|W_o\|_1}} \label{eq:wo_adjust}
\end{align}

% 然而在经典transformer架构中，由于wq与wk的维度相同，当矩阵足够大时，其随机正态初始化得到的矩阵的范数也几乎相同，WISCA的调整很微小。可以证明当两层参数量相同时，随机初始化（高斯）使得两层的L1/L2范数近似相等。
In classical Transformer architectures, \( W_q \) and \( W_k \) typically share the same dimensionality. Under Gaussian initialization, their \( L_1 \)/\( L_2 \)-norms become approximately equal as the parameter count grows, rendering naive WISCA adjustments negligible (Appendix Theorem~\ref{Norm_Convergence}).  

% 但是在使用GQA架构时，例如每8个head共享一组kv，故q的参数量为k的8倍。所以WISCA调整中，scaling-ratio近似等于0.3535（\sqrt{1}{8}），调整前后的变化很大。但是当模型结构为gqa时，是否应当对计算得到的L1范数取均值呢？下面是分析：
% 前置知识：模拟qk层计算loss：L=(XY-1)**2，在XY=K任意双曲线上，X=Y时为路径上最flatten的点。所以QK参数越相近越好。
% 模拟gqa共享Y：
% X = [X1,X2]
% Y = [Y1,Y1]
% (X1*Y1+X2*Y1-1)^2[Y1*(X1+X2)-1]^2Y1 = (X1+X2)时，更加平滑对应在网络参数中，X表示q，Y表示k。由于是gqa，所以wk的参数量少。对应在上述共享的Y1。所以不是应该Y1=X1=X2，而是Y1=(X1+X2)。
% 所以WISCA在GQA上同样是需要调整使得WQ和WK的范数相等，而在GQA上的调整幅度比传统MHA大。

Another essential application of WISCA is for the GQA-based architectures~\citet{ainslie2023gqa} (e.g., \( g \) query heads sharing one key/value group), \( W_q \) has \( g\times \) more parameters than \( W_k \), leading to a significant WISCA scaling ratio of \( \sqrt{1/g} \). For GQA, we propose \textbf{group-averaged normalization}:  

\begin{itemize}
    \item \textbf{Loss Flatness Principle}: For \( \mathcal{L} = (\mathbf{QK} - 1)^2 \), flatness is maximized when $\mathbf{Q}=\mathbf{K}$.  
    \item \textbf{GQA Simulation}:  
    \begin{align*}
        \textbf{Q} &= [Q_1, Q_2, \dots, Q_g], \\ 
        \quad \textbf{K} &= [K_1, K_1, \dots, K_1], \\
        \mathcal{L} &= \left[ K_1 \textstyle \sum_{i=1}^g Q_i - 1 \right]^2
    \end{align*}
    Optimal flatness occurs when \( K_1 = \sum_{i=1}^g Q_i \), not \( K_1 = Q_1 = Q_2 = \dots = Q_g \).  
\end{itemize}

\textbf{Conclusion}: WISCA adjustments for GQA enforce \( \|W_q\|_1 = \|W_k\|_1 \times g \), yielding \textbf{larger scaling effects} than in MHA. Similarly, WISCA can be extended to matrix multiplication with unequal matrices, such as the A and B matrices of LoRA.

% 综上，我们证明了WISCA的思想：调整wq与wk范数相等，调整wv与wo范数相等，并保持调整前后的模型输出效果不变。并且证明了WISCA在GQA等权重结构不均衡的架构上效果更明显。在实验中我们采用了两种WISCA方案，分别是针对tensor的调整和针对channeal的调整，针对channeal的调整方案可以看作针对tensor调整方案的更细粒度设计。

We have demonstrated both theoretically and empirically that WISCA achieves its goals through core principles in Transformer architecture: Adjust \( \|W_q\|_1 = \|W_k\|_1 \) and \( \|W_v\|_1 = \|W_o\|_1 \), while preserving model outputs. The impact of WISCA is more pronounced in architectures with imbalanced parameter groups, such as Grouped Query Attention (GQA), where scaling (e.g., \( \sqrt{1/g} \)) leads to significant tuning.

In experiments, we implement two WISCA variants:  
\begin{itemize}
    \item \textbf{Tensor-wise WISCA}: Global scaling applied to entire weight matrices (Equations \ref{eq:wq_adjust}--\ref{eq:wo_adjust}).  
    \item \textbf{Channel-wise WISCA}: A finer-grained variant that applies scaling at the channel level, enhancing flexibility in parameter optimization.  
\end{itemize}

\section{Experiments}

% 我们设计了针对wq与wk和wv与vo的tensor-wise的WISCA和channel-wise的WISCA，并且在开源的llama结构和qwen结构以及流行的MoE结构上进行了模型预训练对比实验，对比了不同WISCA策略对origin策略的训练收敛效果以及测评效果。实验表明WISCA对模型的训练效果有较明显的优化。
We designed tensor-wise and channel-wise WISCA for $W_q$/$W_k$ and $W_v$/$W_o$, and conducted pre-training comparative experiments on open-source Llama, Qwen, and popular Mixture-of-Experts (MoE) architectures~\cite{jacobs1991adaptive}. The experiments evaluated the training convergence and downstream task performance of different WISCA strategies compared to the original approach. Results demonstrate that WISCA significantly improves model training performance across all tested architectures.

\subsection{Tensor-wise and Channel-wise WISCA}

We implement four WISCA variants targeting different components of the attention mechanism: (1) adjustments on $W_q$/$W_k$ for attention score computation, covering tensor-wise (Figure~\ref{fig:wisca-qk-mha}) and channel-wise (Figure~\ref{fig:wisca-qk-gqa}) optimizations; and (2) adjustments on $W_v$/$W_o$ for output projection, also including tensor-wise (Figure~\ref{fig:wisca-vo-mha}) and channel-wise (Figure~\ref{fig:wisca-vo-gqa}) variants.

\begin{figure}[!htbp]
\begin{center}
\centerline{\includegraphics[width=0.9\columnwidth]{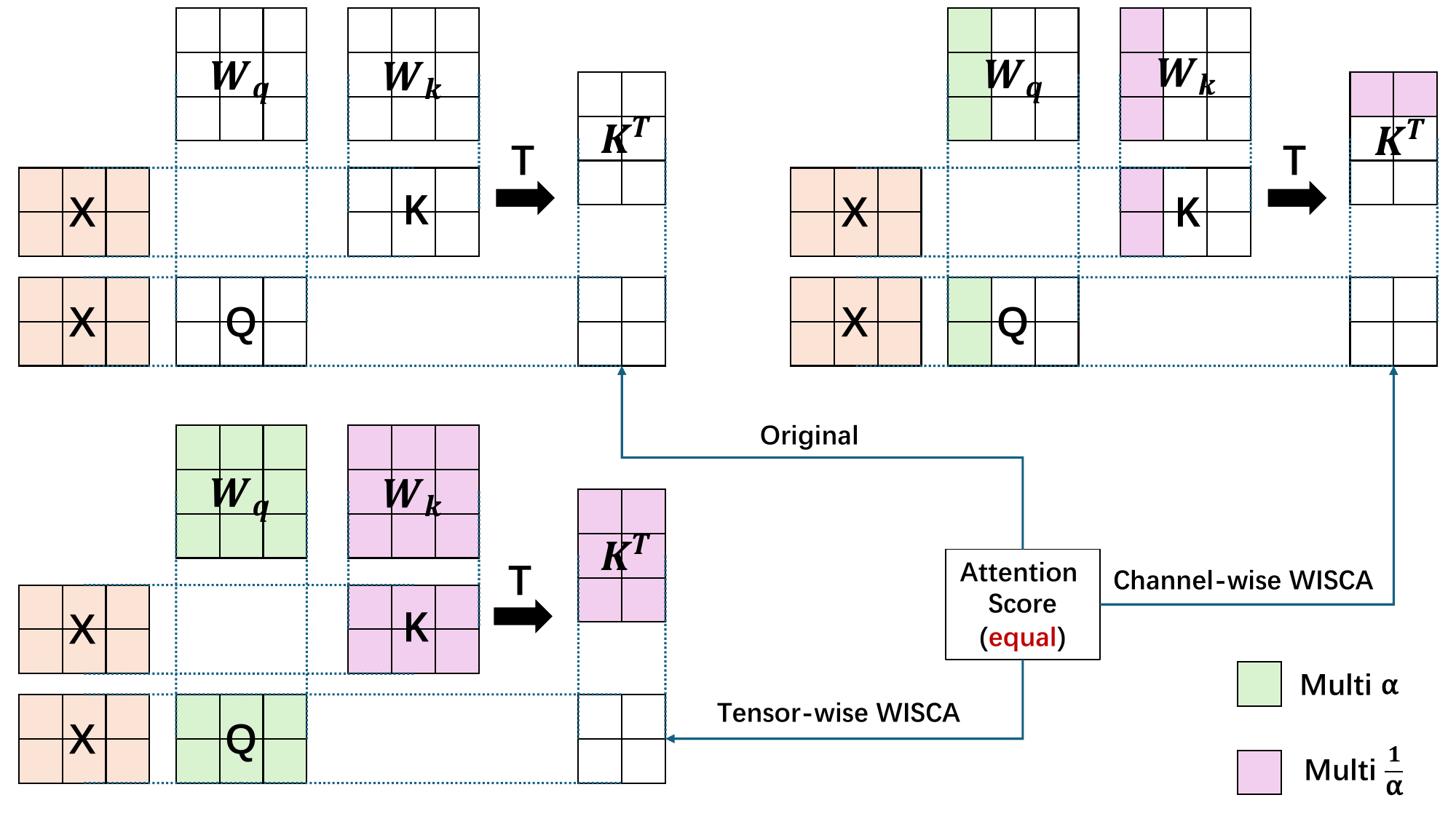}}
\caption{
% 在wq与wk上的tensor-wise 的WISCA和channel-wise的WISCA与原始的计算attention score方式的对比。三种方式计算出的attention score完全相等，但是模型权重不同。
\textbf{Comparison of Attention Score Computation Methods.} 
Original attention mechanism (top-left), tensor-wise WISCA (bottom-down), and channel-wise WISCA (top-right) on $W_q$/$W_k$ adjustments. All methods produce \textbf{identical attention scores} but exhibit \textbf{distinct weight patterns}. 
}
\label{fig:wisca-qk-mha}
\end{center}
\vskip -0.3in
\end{figure}

As shown in Figure~\ref{fig:wisca-qk-mha}, both QK-WISCA variants retain the original attention scores while redistributing weight magnitudes, with channel-wise scaling offering finer-grained control than tensor-wise. Similarly, Figure~\ref{fig:wisca-vo-mha} shows that $W_v$/$W_o$ adjustments preserve self-attention outputs despite altered weight patterns, mirroring the pattern observed in attention-score optimization.  

For Grouped Query Attention (GQA) architectures, we further validate the adaptability of channel-wise WISCA. Figures \ref{fig:wisca-qk-gqa} and \ref{fig:wisca-vo-gqa} demonstrate that GQA's dimensional asymmetry requires group-aware scaling in channel-wise WISCA.

\begin{figure}[!htbp]
\begin{center}
\centerline{\includegraphics[width=0.9\columnwidth]{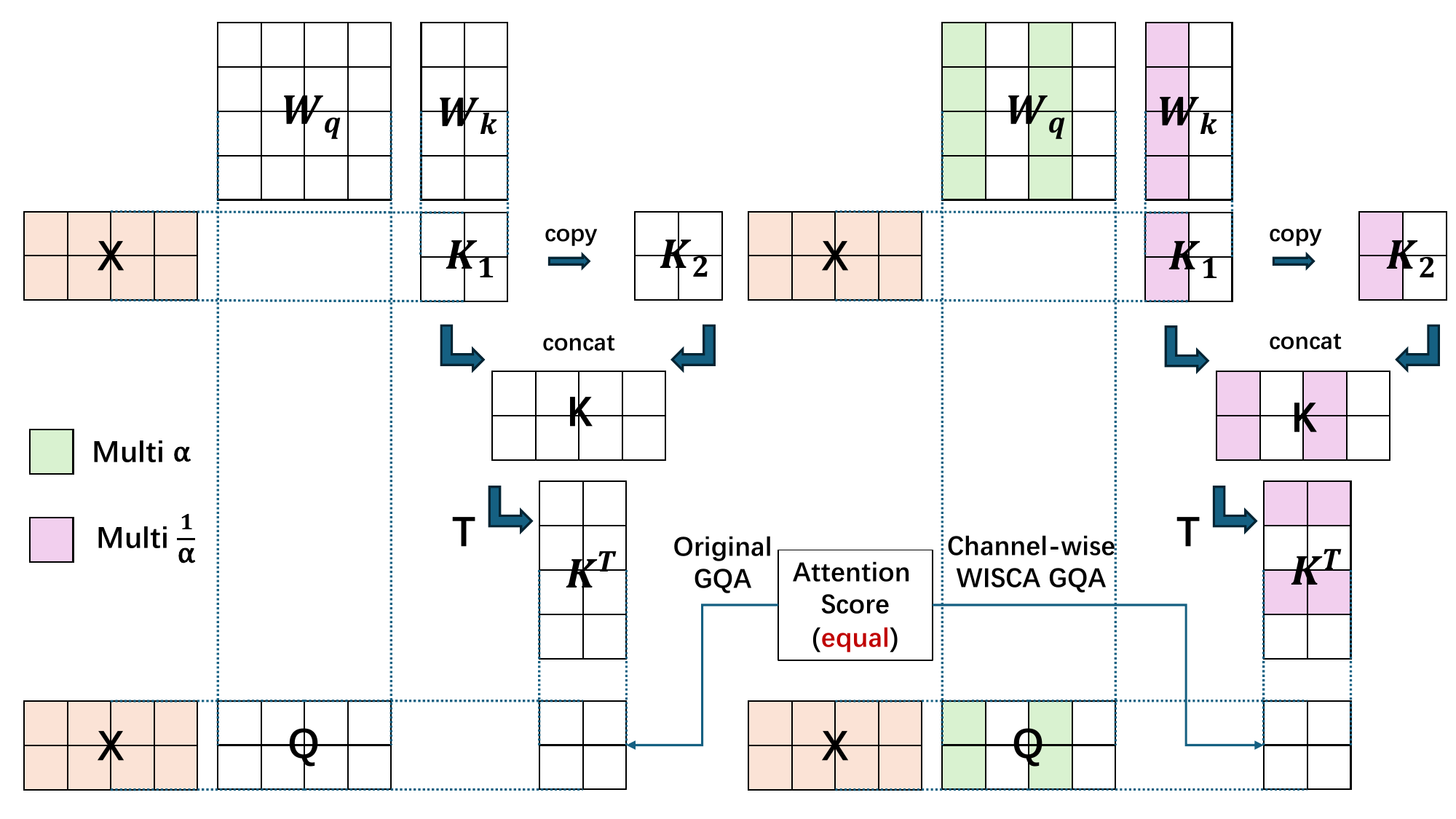}}
\caption{
% 在GQA架构的模型中，在wq与wk上的channel-wise的WISCA（右图）与原始的计算attention score方式（左图）的对比。两种方式计算出的attention score完全相等，但是模型权重不同。GQA使得Wk与Wq维度不同，channel-wise的WISCA需要利用GQA的特性进行设计，而tensor-wise的WISCA策略则与MHA相同（省略）。
\textbf{Channel-wise WISCA in GQA Architecture.} 
Comparison between original attention (left) and channel-wise WISCA (right) for $W_q$/$W_k$ adjustments in Grouped Query Attention (GQA). Both methods produce \textbf{identical attention scores} but exhibit \textbf{distinct weight patterns}. GQA introduces dimensional asymmetry between $W_q$ and $W_k$, requiring channel-wise WISCA to adapt to group-wise scaling (tensor-wise WISCA remains identical to MHA, omitted here).
}
\label{fig:wisca-qk-gqa}
\end{center}
\vskip -0.3in
\end{figure}

% 在主流LLMs上的收敛效果
\subsection{Convergence Effect on Mainstream LLMs}

% 为了验证wisca在实际训练中的效果，我们在tinyllama, qwen2-1.5B和qwen1.5-moe三种结构的模型上使用tinystories数据集进行了pretrain实验。

% 我们使用\ref{fig:wisca-qk-gqa}和\ref{fig:wisca-vo-gqa}中针对tensor的wisca方法与origin训练方法进行了对比试验，所有的wisca步骤均为每250iter执行一次，并且在统一模型上的四组实验的其他超参数完全一致（例如lr, optimizer 等待），实验效果见\ref{tab:model_comparison_pretrain}
% 实验描述段落
To validate WISCA's effectiveness in real training scenarios, we conduct pre-training experiments on three architectures (\textbf{TinyLlama}~\cite{zhang2024tinyllama}, \textbf{Qwen2-1.5B}~\cite{team2024qwen2}, and \textbf{Qwen1.5-MoE}~\cite{qwen_moe}) using the TinyStories~\cite{eldan2023tinystories} dataset. We compare tensor-wise WISCA strategies (\ding{172} QK\_TEN and \ding{173} VO\_TEN shown in Figure \ref{fig:wisca-qk-mha} and Figure \ref{fig:wisca-vo-mha}) against the original training approach, with WISCA applied every 250 iterations while maintaining identical hyperparameters (learning rate, optimizer, batch size, etc.) across all experimental groups. Training loss and test perplexity (PPL) results are summarized in Table \ref{tab:model_comparison_pretrain}.

\begin{table}[h]
\centering
\caption{Training losses and test PPL scores for different strategies on various models}
% \vskip 0.15in
\begin{center}
\begin{small}
\begin{sc}
\begin{tabular}{cccc}
\toprule
Model & Strategy & Train Loss & Test PPL \\
\midrule
\multirow{4}{*}{TinyLlama} & origin & 1.3193 & 3.78 \\
 & \ding{172} QK\_TEN & 1.2849 & 3.66 \\
 & \ding{173} VO\_TEN & 1.3123 & 3.75 \\
 & \ding{172}+\ding{173} & \textbf{1.2749} & \textbf{3.62} \\
\midrule
\multirow{4}{*}{Qwen2-1.5B} & origin & 1.355 & 3.96 \\
 & \ding{172} QK\_TEN & 1.3417 & 3.91 \\
 & \ding{173} VO\_TEN & 1.3507 & 3.94 \\
 & \ding{172}+\ding{173} & \textbf{1.3336} & \textbf{3.88} \\
\midrule
\multirow{4}{*}{Qwen1.5-MoE} & origin & 1.5497 & 4.76 \\
 & \ding{172} QK\_TEN & 1.519 & 4.62 \\
 & \ding{173} VO\_TEN & 1.5408 & 4.72 \\
 & \ding{172}+\ding{173} & \textbf{1.5141} & \textbf{4.60} \\
\bottomrule
\end{tabular}
\end{sc}
\end{small}
\end{center}
\label{tab:model_comparison_pretrain}
\vskip -0.2in
\end{table}

% 在验证wisca在预训练任务中的效果中，所有的实验效果均展示了wisca对模型收敛效果的优化，并且在数据集中的表现为针对两个神经网络模块的wisca合并后的效果会更好。
Across all pre-training tasks, the experimental results demonstrate that WISCA consistently improves model convergence behavior. Notably, the combined application of QK-WISCA and VO-WISCA yields superior performance compared to individual module adjustments, exhibiting a \textbf{synergistic effect} on the target dataset. This improvement pattern remains consistent across all tested architectures.

% LLMs上的测评效果
\subsection{Evaluation Results on LLMs}

% 为了验证wisca训练模型的Zero-shot evaluation效果，我们使用wikipedia.en数据集从头训练了1.1B的llama后，使用EleutherAI/lm-evaluation-harness开源仓库进行了测评，测评效果见\ref{tab:llama_metrics}
To evaluate WISCA's zero-shot generalization capabilities, we trained a 1.1B-parameter Llama model from scratch on 1.4B tokens from Wikipedia.en, followed by comprehensive evaluation using the EleutherAI LM Evaluation Harness~\cite{eval-harness}.

\begin{table}[h]
\caption{Evaluation Metrics of llama-1.1B Trained on 1.4B Tokens from Wikipedia.en. Performance comparison of different attention optimization methods. All metircs seen in table \ref{tab:llama_metrics_all}}
\label{tab:llama_metrics}
\begin{center}
\begin{small}
\begin{sc}
\setlength{\tabcolsep}{3pt}
\begin{tabular}{lccccc}
\toprule
 version & BoolQ↑ & ARC-c↑ & PIQA↑ & WinoG↑ & avg↑ \\
\midrule
 origin & 0.3838 & 0.1741 & 0.5288 & 0.5004 & 0.3968 \\
\midrule
 \ding{172} qk\_ten & 0.3810 & 0.1843 & 0.5332 & 0.4807 & 0.3948 \\
 \ding{173} qk\_row & 0.3887 & 0.1817 & 0.5370 & 0.5091 & 0.4041 \\
 \ding{174} vo\_ten & 0.4015 & 0.1852 & 0.5294 & 0.4957 & 0.4030 \\
 \ding{175} vo\_row & 0.3817 & 0.1749 & 0.5386 & 0.4988 & 0.3985 \\
\midrule
 \ding{172}+\ding{174} & \textbf{0.5214} & 0.1869 & 0.5413 & 0.4980 & \textbf{0.4369} \\
 \ding{173}+\ding{174} & 0.4483 & \textbf{0.2014} & 0.5305 & 0.5059 & 0.4215 \\
 \ding{172}+\ding{175} & 0.4226 & 0.1783 & 0.5310 & 0.5075 & 0.4099 \\
  \ding{173}+\ding{175} & 0.3994 & 0.1724 & \textbf{0.5430} & \textbf{0.5170} & 0.4080 \\
\midrule
 \ding{172}+\ding{174}(init) & 0.4703 & 0.1860 & 0.5299 & 0.4964 & 0.4207 \\
 \ding{173}+\ding{174}(init) & 0.3960 & 0.1792 & 0.5386 & 0.5036 & 0.4044 \\
 \ding{172}+\ding{175}(init) & 0.4312 & 0.1766 & 0.5308 & 0.4949 & 0.4083 \\
 \ding{173}+\ding{175}(init) & 0.3917 & 0.1817 & 0.5337 & 0.5043 & 0.4029 \\
\bottomrule
\end{tabular}
\end{sc}
\end{small}
\end{center}
\vskip -0.2in
\end{table}

% 在端到端测评实验中，我们除了验证针对tensor和channel的WISCA，以及它们之间的四种combo效果，由于第一步wisca对模型在landscape位置效果的影响最大，所以还验证了仅在初始化阶段应用wisca的效果。
In the end-to-end evaluation experiments, we validated both tensor/channel-wise WISCA strategies and their four combinatorial variants. Given the pronounced impact of initial WISCA adjustments on model positioning within the loss landscape, we additionally evaluated strategies where WISCA was applied only at initialization (marked as "(init)" in Table \ref{tab:llama_metrics}).

% 在测评实验中表明，combo的wisca的效果明显优于单独的wisca优化，并且如果只在初始化阶段做wisca也同样有明显的优化效果，但是不如在训练过程中间断的进行wisca操作。所以如果需要不损失训练过程中的计算资源来进行wisca计算，可以仅在初始化阶段进行。
The evaluation experiments demonstrate that combinatorial WISCA strategies (\ding{172}+\ding{174}, etc.) significantly outperform individual optimizations, with the best combination achieving a 10.1\% average improvement over the baseline (Table \ref{tab:llama_metrics}), highlighting synergistic effects between QK-WISCA and VO-WISCA adjustments.

While initialization-only WISCA retains an average of 97\% of the full combinatorial performance, it suggests that periodic adjustments provide incremental refinements to weight patterns. For resource-constrained scenarios, initialization-only WISCA offers a computationally efficient alternative (no training overhead) while preserving most performance gains.

\begin{table}[h]
\caption{Zero-shot evaluation results on different training steps on llama\_moe-5B-A0.8B.All metircs seen in table \ref{tab:zero-shot-results_all}}
\label{tab:zero-shot-results}
\begin{center}
\begin{small}
\begin{sc}
\setlength{\tabcolsep}{4pt}
\begin{tabular}{ccccccc}
\toprule
Steps & \multicolumn{3}{c}{Metrics AVG↑} & \multicolumn{3}{c}{PPL(wikitext2)} \\
\cmidrule(lr){2-4} \cmidrule(lr){5-7}
(Tokens) & origin & \ding{172}+\ding{174} & +(\%) & origin & \ding{172}+\ding{174} & +(\%) \\
\midrule
1.43B & 35.56 & 37.07 & \textbf{4.25} & 70.93 & 68.11 & \textbf{3.98} \\
2.86B & 37.88 & 38.23 & \textbf{0.92} & 48.17 & 46.25 & \textbf{3.98} \\
4.29B & 39.28 & 39.49 & \textbf{0.53} & 40.72 & 39.17 & \textbf{3.82} \\
5.72B & 39.83 & 40.68 & \textbf{2.13} & 35.48 & 34.85 & \textbf{1.77} \\
7.15B & 40.13 & 40.33 & \textbf{0.50} & 33.29 & 32.77 & \textbf{1.56} \\
8.58B & 40.31 & 40.69 & \textbf{0.94} & 31.16 & 30.80 & \textbf{1.15} \\
10.0B & 39.39 & 40.72 & \textbf{3.38} & 29.45 & 29.21 & \textbf{0.82} \\
\bottomrule
\end{tabular}
\end{sc}
\end{small}
\end{center}
\vskip -0.2in
\end{table}

% 为了验证wisca在模型训练过程中不同收敛阶段下对模型优化的影响，我们在llama-moe模型中进行了10B tokens的训练，并对训练中存储的7次ckpt进行了指标评测，评测结果见\ref{tab:zero-shot-results}
To analyze WISCA's impact across training stages, we train a Llama-MoE-5B-A0.8B model for 10B tokens, evaluating seven intermediate checkpoints. Zero-shot performance metrics and perplexity scores are reported in Table \ref{tab:zero-shot-results}, demonstrating WISCA's consistent optimization benefits throughout training.

% 实验表明由于在init阶段wisca的调整最大，调整前后模型的光滑度差别最大，所以在前期的提升效果最明显。在测评指标中，尽管不同训练步骤下有所波动，但wisca均为正提升。
The results demonstrate that, due to the extensive tuning of WISCA during initialization, the most significant improvements are achieved early in training, when the model's loss graph smoothness undergoes the largest shift.

While metric fluctuations occur across training steps (e.g., 0.50–3.38\% accuracy variations), WISCA consistently outperforms the baseline in all evaluations. The sustained positive trends—particularly the average 1.81\% metrics improvement and 2.44\% perplexity reduction across 10B tokens—validate its robustness as a general optimization strategy. 

\subsection{Benefits on LoRA}

\textbf{Low-Rank Adaptation (LoRA)}~\cite{hu2022lora} is a parameter-efficient fine-tuning (PEFT) paradigm widely adopted for large language models (LLMs). The core hypothesis posits that weight updates (\(\Delta W\)) during model adaptation exhibit a low \emph{intrinsic rank}, implying they can be approximated by low-dimensional structures rather than full-rank matrices.

Given a pre-trained weight matrix \( W \in \mathbb{R}^{m \times n} \), LoRA introduces two trainable low-rank matrices \( A \in \mathbb{R}^{m \times r} \) and \( B \in \mathbb{R}^{r \times n} \), where \( r \ll \min(m, n) \). The adapted weight is computed as:
\begin{equation}
    W' = W + \Delta W = W + A \cdot B
\end{equation}
During fine-tuning, \emph{only} \( A \) and \( B \) are updated, while the original \( W \) remains frozen. 

% 当m\notequal{n}时，A矩阵与B矩阵参数量不相等，并且在原始lora中，B矩阵为0矩阵，||A||\notequal{||B||}，可以在lora训练中间执行wisca操作。

% 我们使用llama-factory对DeepSeek-V2-Lite-Chat进行lora微调，在alpaca与metamath两个数据集中进行了lora，pissa和rs-lora三种微调策略中对比了原始微调效果与嵌入wisca的训练效果，实验结果见\ref{tab:wisca-loss-lora}。

When \( m \neq n \), the parameter counts of matrices \( A \) and \( B \) become imbalanced. In standard LoRA, the \( B \) matrix is initialized as zero (\(\|A\| \neq \|B\|\)), enabling dynamic WISCA operations during training. We fine-tuned the DeepSeek-V2-Lite-Chat model on the Alpaca and MetaMath datasets using LoRA, PISSA~\cite{meng2024pissa}, and RS-LoRA~\cite{kalajdzievski2023rank} strategies via Llama-Factory~\cite{zheng2024llamafactory}, comparing vanilla fine-tuning with WISCA-enhanced training (Table \ref{tab:wisca-loss-lora}). Results validate its universal optimization capability in parameter-efficient fine-tuning (PEFT) scenarios.

The \textbf{Alpaca}~\cite{alpaca} dataset is a widely-used benchmark for instruction tuning, generated by fine-tuning LLaMA on self-instruct-style human-AI interactions to improve language models' ability to follow diverse commands. \textbf{MetaMath}~\cite{yu2023metamath} is a specialized mathematical reasoning dataset containing 8.4k problem-solution pairs across arithmetic, algebra, and geometry, designed to evaluate and enhance models' multistep reasoning capabilities. 

\begin{table}[ht]
\caption{
Training Loss Comparison: Vanilla vs. WISCA enhanced SFT fine-tuning for DeepSeek-V2-Lite-Chat~\cite{deepseekv2} on Alpaca and MetaMath. All experiments were trained with 3 epochs. 
}
\label{tab:wisca-loss-lora}
\begin{small}
\begin{sc}
\setlength{\tabcolsep}{7.5pt}
\centering
\begin{tabular}{lrrrr}
\toprule
\multirow{2}{*}{Method} & \multicolumn{2}{c}{Alpaca} & \multicolumn{2}{c}{MetaMath} \\
\cmidrule(lr){2-3} \cmidrule(lr){4-5}
& Origin & WISCA & Origin & WISCA \\
\midrule
LoRA    & 0.8602 & \textbf{0.8532} & 0.0779 & \textbf{0.0770} \\
PISSA   & 0.6227 & \textbf{0.6176} & 0.0692 & \textbf{0.0688} \\
RS-LoRA & 0.6773 & \textbf{0.6760} & 0.0744 & \textbf{0.0726} \\
\bottomrule
\end{tabular}
\end{sc}
\end{small}
\vskip -0.1in
\end{table}

% 由于原始lora的B矩阵初始化为全0矩阵，所以无法在初始化时进行wisca操作，\ref{tab:wisca-loss-lora}实验中在总iterations的三分之一处执行的wisca操作。
Since the B matrix in vanilla LoRA is initialized as a zero matrix, WISCA cannot be applied during initialization. In our experiments (Table \ref{tab:wisca-loss-lora}), WISCA operations were performed at one-third of the total training iterations.

\subsection{Benefits on Eagle}

% 投机采样是一种用于大模型推理的策略，eagle为投机采样中影响力最大的工作。投机采样通过训练小模型来生成大模型的结果，大模型负责审核小模型的输出。

% 我们通过使用eagle训练LLaMA 3.1 Instruct 8B的小模型，eagle的小模型为一层transformer块，同样包含一层attention部分，我们对eagle的训练部分同样应用了wisca策略，评价指标为接收tokens个数。实验中tree_mode, depth=6, topk=10；chain_mode, depth=5。

\textbf{EAGLE}~\cite{li2024eagle} is a state-of-the-art speculative sampling framework that significantly accelerates LLMs' inference. It employs a draft model (typically a small neural network) to generate candidate token sequences, which are then validated in parallel by the target LLM.

We implement EAGLE by training a draft model for \textbf{LLaMA 3.1 Instruct 8B}~\cite{grattafiori2024llama}. The draft model architecture consists of a single Transformer layer with self-attention, incorporating our proposed WISCA strategy during training to optimize weight patterns.

\textbf{Dataset}: We use the \textbf{ShareGPT} dataset, a collection of 90K high-quality human-AI conversations spanning diverse domains (e.g., coding, reasoning, open-ended dialogue). Each example contains multi-turn interactions, making it suitable for training instruction-following draft models.

\begin{table}[ht]
\caption{
EAGLE-1 Performance: Average Accepted Tokens and Training Metrics. Tree/Chain modes use depth=6/top-k=10 and depth=5, respectively.
}
\label{tab:eagle-wisca}
\begin{small}
\begin{sc}
\setlength{\tabcolsep}{7pt}
\begin{tabular}{lrrrr}
\toprule
\multirow{2}{*}{\textbf{Metric}} & \multicolumn{2}{c}{\textbf{Original}} & \multicolumn{2}{c}{\textbf{WISCA}} \\
\cmidrule(lr){2-3} \cmidrule(lr){4-5}
& Tree & Chain & Tree & Chain \\
\midrule
MT-Bench       & 3.001 & 1.528 & \textbf{3.009} & \textbf{1.593} \\
GSM8K          & 3.234 & \textbf{1.818} & \textbf{3.271} & 1.809 \\
HumanEval      & 3.708 & 2.103 & \textbf{3.715} & \textbf{2.110} \\
Alpaca-Eval    & 2.874 & 1.462 & \textbf{2.900} & \textbf{1.478} \\
\midrule
Train Loss     & \multicolumn{2}{c}{0.7168} & \multicolumn{2}{c}{\textbf{0.7113}} \\
Val Loss       & \multicolumn{2}{c}{0.7359} & \multicolumn{2}{c}{\textbf{0.7312}} \\
Train Acc (\%) & \multicolumn{2}{c}{79.28} & \multicolumn{2}{c}{\textbf{79.49}} \\
Val Acc (\%)   & \multicolumn{2}{c}{77.03} & \multicolumn{2}{c}{\textbf{77.11}} \\
\bottomrule
\end{tabular}
\end{sc}
\end{small}
\vskip -0.1in
\end{table}

% \ref{tab:eagle-wisca}中的实验，wisca在draft model的[wq, wk]和[wv, wo]中应用tensor维度的wisca。实验数据表明wisca在eagle训练也有一定的优化，在端到端测评的结果中wisca基本在tree模型和chain模型均优于原始训练策略。
In experiments with the EAGLE draft model (Table \ref{tab:eagle-wisca}), we applied tensor-wise WISCA to the $[w_q, w_k]$ and $[w_v, w_o]$ modules. The results show that WISCA consistently improves training efficiency, with the WISCA-enhanced model outperforming the baseline in end-to-end evaluations in both \emph{tree} and \emph{chain} generation modes. This validates the generalizability of WISCA’s weight pattern optimization in speculative sampling frameworks.

\section{Related Work}
\textbf{Loss Landscape Optimization.} 
Sharpness-Aware Minimization (SAM)~\cite{foret2020sharpness} explicitly minimizes loss sharpness, albeit at the cost of approximately double the computation per parameter update.
Stochastic Weight Averaging (SWA)~\cite{izmailov2018averaging} improves generalization by weight averaging, but requires extensive training. 
WISCA implicitly smooths the loss landscape through weight pattern adjustments , achieving efficiency gains with minimal overhead.

\textbf{Weight Balancing and Normalization.} 
Query key normalization (QKN)~\cite{henry2020query} improves the stability of transformer training by applying L2 normalization to the query (Q) and key (K) matrices, transforming their dot products into cosine similarities. 
Weight normalization~\cite{salimans2016weight} reparameterizes weights to decouple direction and magnitude. 
Unlike these methods, WISCA enforces \emph{functional equivalence} while balancing weight norms, enabling dynamic transitions between equivalent models without architectural changes.

\textbf{Model Equivalence and Weight Patterns.} The Lottery ticket hypothesis~\cite{frankle2018lottery} identifies performant sparse subnets, whereas permutation symmetry reveals equivalence under weight permutations. 
WISCA leverages these insights to navigate the loss landscape via equivalent model transitions, prioritizing flat minima for enhanced generalization.

\textbf{Parameter-Efficient Fine-Tuning (PEFT).} 
Low-Rank Adaptation (LoRA)~\cite{hu2022lora} reduces trainable parameters by decomposing weight updates into low-rank matrices. 
Recent variants like PiSSA~\cite{meng2024pissa} further improve efficiency by adapting only the principal singular values and vectors of the weight matrices, achieving faster convergence than LoRA.
WISCA can be integrated with PEFT methods such as LoRA and PiSSA to optimize weight patterns in their low-rank matrices, enhancing both training stability and task-specific adaptation without introducing additional parameters.

\textbf{Speculative Decoding.} 
EAGLE~\cite{li2024eagle} accelerates LLM inference through a small draft model. 
Our work improves EAGLE by applying WISCA to the draft layer training, improving token acceptance rates through better weight patterns.

\section{Conclusion}

In this work, we introduced WISCA, a training optimization strategy that reshapes weight patterns to enhance model robustness without architectural changes. WISCA improves generalization by achieving flatter minima through equivalent model transitions. Experiments across various architectures, such as GQA, MoE, and LoRA, demonstrated WISCA's effectiveness in reducing training perplexity, enhancing zero-shot performance, and increasing inference efficiency. This approach suggests potential for broader application and future exploration of equivalent model strategies across different neural networks.

\section{Limitations}
Despite the promising results, WISCA has several limitations that warrant further investigation:

Architectural Constraints: While WISCA effectively optimizes weight patterns in Transformer-based architectures, its applicability to other neural network structures, such as convolutional or recurrent networks, remains unexplored. Further research is needed to adapt WISCA to diverse architectures.

Experimental Scope: The experiments conducted focus primarily on specific datasets and model configurations. The generalizability of WISCA across different data domains and larger-scale models requires comprehensive evaluation.

\nocite{zheng2023judging}
\nocite{chen2021evaluating}
\nocite{cobbe2021training}
\nocite{clark2019boolq}
\nocite{allenai:arc}
\nocite{Bisk2020}
\nocite{thrush_and_ross2022winoground}
\nocite{zellers2019hellaswag}
\nocite{OpenBookQA2018}
\nocite{huang2023ceval}
\nocite{vaswani2017attention}
\nocite{shazeer2019fast}
\nocite{leviathan2023fast}
\nocite{huo2025c2t}
\nocite{qin2025maskprune}
\nocite{xu2023parameter}
\nocite{keskar2016large}
\nocite{touvron2023llama}
\nocite{merity2016pointer}
\nocite{team2025longcat}

\bibliography{custom}

\appendix

\label{sec:appendix}

\section{Norm Convergence in Classical Transformers}

\begin{theorem} 
\label{Norm_Convergence}
For Gaussian-initialized matrices \( \mathbf{W}_q, \mathbf{W}_k \in \mathbb{R}^{m \times n} \) with i.i.d. entries \( \mathcal{N}(0, \sigma^2) \), the \( L_1 \)- and \( L_2 \)-norms satisfy:  
\[
\frac{\|\mathbf{W}_q\|_1}{\|\mathbf{W}_k\|_1} \xrightarrow{p} 1 \quad \text{and} \quad \frac{\|\mathbf{W}_q\|_2}{\|\mathbf{W}_k\|_2} \xrightarrow{p} 1 \quad \text{as } mn \to \infty,
\]  
where \( \xrightarrow{p} \) denotes convergence in probability.  
\end{theorem}  
\begin{proof}  
1. \textbf{\( L_2 \)-Norm (Frobenius Norm)}:  
   Let \( \|\mathbf{W}\|_2^2 = \sum_{i,j} W_{ij}^2 \). For Gaussian entries:  
   \[
   \mathbb{E}[\|\mathbf{W}\|_2^2] = mn\sigma^2, \quad \text{Var}(\|\mathbf{W}\|_2^2) = 2mn\sigma^4.
   \]  
   By Chebyshev’s inequality:  
   \[
   P\left(\left|\frac{\|\mathbf{W}\|_2^2}{mn\sigma^2} - 1\right| > \epsilon\right) \leq \frac{2mn\sigma^4}{(mn\sigma^2\epsilon)^2} = \frac{2}{\epsilon^2 mn} \to 0.
   \]  
   Thus, \( \|\mathbf{W}_q\|_2 / \|\mathbf{W}_k\|_2 \to 1 \).  

2. \textbf{\( L_1 \)-Norm}:  
   Let \( \|\mathbf{W}\|_1 = \sum_{i,j} |W_{ij}| \). For \( W_{ij} \sim \mathcal{N}(0, \sigma^2) \):  
   \[
   \mathbb{E}[|W_{ij}|] = \sigma\sqrt{\frac{2}{\pi}}, \quad \text{Var}(|W_{ij}|) = \sigma^2\left(1 - \frac{2}{\pi}\right).
   \]  
   Summing over \( mn \) terms:  
\begin{align*}
   \mathbb{E}[\|\mathbf{W}\|_1] &= mn\sigma\sqrt{\frac{2}{\pi}}, \\
   \text{Var}(\|\mathbf{W}\|_1) &= mn\sigma^2\left(1 - \frac{2}{\pi}\right).
\end{align*}
   Again by Chebyshev:  
   \[
   P\left(\left|\frac{\|\mathbf{W}\|_1}{mn\sigma\sqrt{2/\pi}} - 1\right| > \epsilon\right) \leq \frac{1 - 2/\pi}{\epsilon^2 mn} \to 0.
   \]  
   Hence, \( \|\mathbf{W}_q\|_1 / \|\mathbf{W}_k\|_1 \to 1 \).  
\end{proof}

\section{Appendix figures}
In this Appendix, we present the WISCA strategy for the \texttt{wv} and \texttt{wo} components of the transformer block, which were not exhibited in the main text. Figure \ref{fig:wisca-vo-mha} illustrates the matrix perspective of tensor-wise WISCA and channel-wise WISCA for a general attention module. Similarly, Figure \ref{fig:wisca-vo-gqa} provides a schematic representation for the attention module with GQA.

\begin{figure}[!htbp]
\begin{center}
\centerline{\includegraphics[width=0.9\columnwidth]{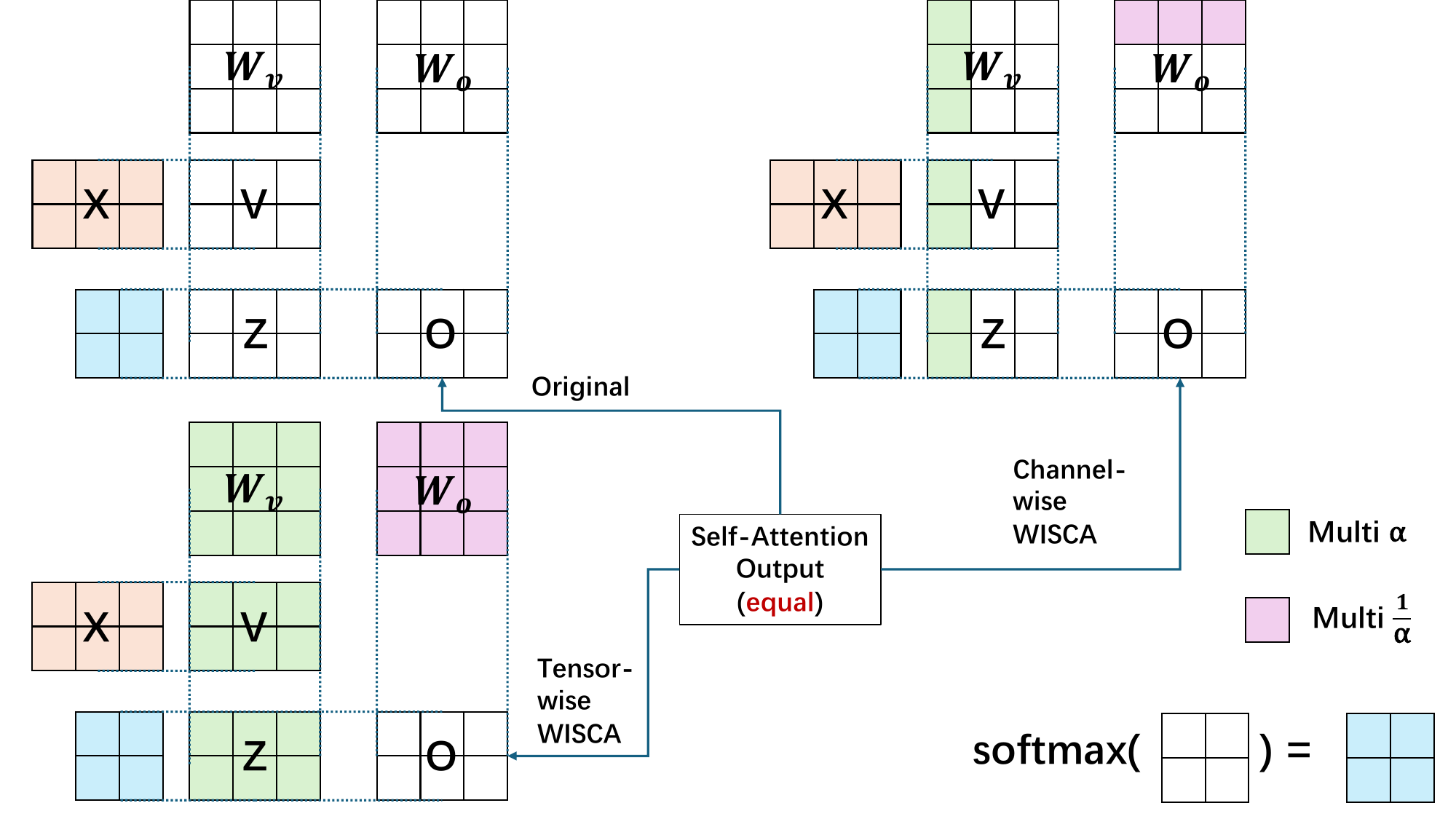}}
\caption{
% 在wv与wo上的tensor-wise 的WISCA和channel-wise的WISCA与原始的计算self-attention output方式的对比。三种方式计算出的结果完全相等，但是模型权重不同。
\textbf{Comparison of Output Projection Methods.} 
Original output projection (left), tensor-wise WISCA (middle), and channel-wise WISCA (right) on $W_v$/$W_o$ adjustments. All methods produce \textbf{identical self-attention outputs} but exhibit \textbf{distinct weight patterns}.
}
\label{fig:wisca-vo-mha}
\end{center}
\vskip -0.3in
\end{figure}

\begin{figure}[!htbp]
\begin{center}
\centerline{\includegraphics[width=0.9\columnwidth]{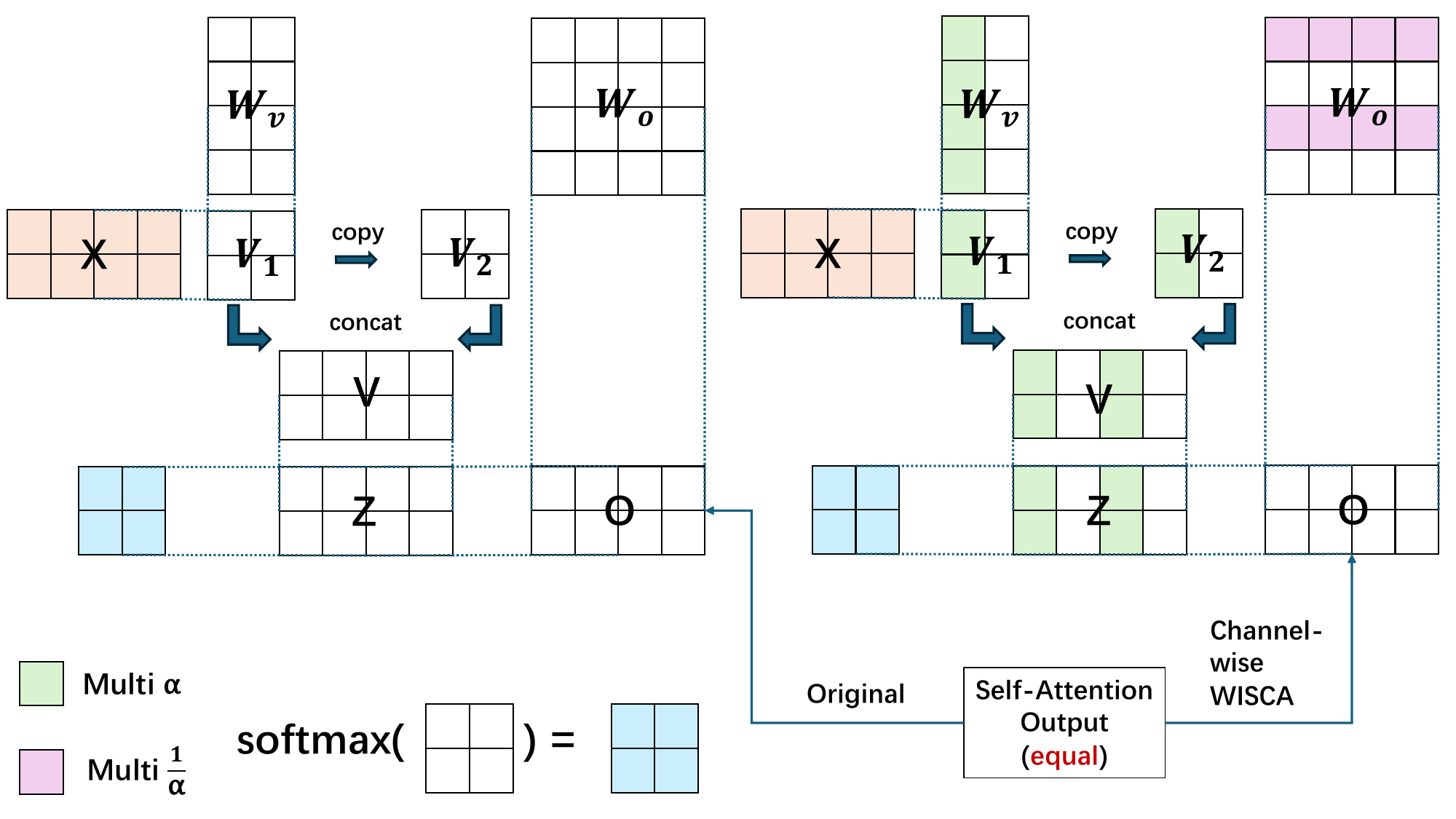}}
\caption{
% 在GQA架构的模型中，在wv与wo上的channel-wise的WISCA（右图）与原始的计算self-attention output方式（左图）的对比。两种方式计算出的self-attention output完全相等，但是模型权重不同。GQA使得Wv与Wo维度不同，channel-wise的WISCA需要利用GQA的特性进行设计，而tensor-wise的WISCA策略则与MHA相同（省略）。
\textbf{Channel-wise WISCA for Output Projection in GQA.} 
Comparison between original output projection (left) and channel-wise WISCA (right) on $W_V$/$W_O$ in Grouped Query Attention (GQA). Both methods yield \textbf{identical self-attention outputs} but with \textbf{distinct weight patterns}. GQA's dimensional asymmetry requires channel-wise WISCA to adapt group-aware scaling, while tensor-wise WISCA remains equivalent to MHA (omitted here).
}
\label{fig:wisca-vo-gqa}
\end{center}
\vskip -0.4in
\end{figure}

\section{Summary of All Metrics}
In the experimental section, we conducted pre-training tasks using the tinyllama and llama-moe models. While the main text includes a streamlined table focusing on zero-shot evaluation metrics, this appendix provides a complete overview of all evaluation metrics for thorough analysis.

The following tables provide detailed metrics gathered from our comprehensive evaluations. Table \ref{tab:llama_metrics_all} breaks down the individual performance metrics across various tests, highlighting the comparative effectiveness of different attention optimization strategies applied to the Llama model. We additionally explore the impact of WISCA at various stages of training in Table \ref{tab:zero-shot-results_all}, showcasing the zero-shot evaluation results for the llama\_moe-5B-A0.8B model across different checkpoints and training steps.

These results underscore the efficacy of our approaches. By presenting all metrics here, we aim to furnish a deeper understanding of the evaluation outcomes in the context of our experimental framework.

\subsection{Proof of Convergence}
Why does the convergence become stable after using WISCA? Here is an intuitive explanation. Considering the landscape $f(Q,K)=(QK-1)^2$ and the current checkpoint is $(Q,K)$, the gradient is $2(QK-1)(K,Q)$. After updating the parameter along the negative gradient direction with learning rate $\eta=\frac{\epsilon}{2(QK-1)}$, the next checkpoint is $(Q-\epsilon K, K-\epsilon Q)$. Now the gradient at this checkpoint is $2((Q-\epsilon K)(K-\epsilon Q)-1)(K-\epsilon Q, Q-\epsilon K)$. We hope the gradient direction varies as small as possible. Therefore, we have 
$
\frac{Q}{K}=\frac{Q-\epsilon K}{K-\epsilon Q}
$
and obtain $K^2=Q^2$, which means $|Q|=|K|$.

Here are the details for above. Consider the minimal analytic example:
$$
L(Q,K)=\tfrac12(QK-C)^2,\qquad C>0.
$$

a) Hessian and sharpness
The Hessian and its trace are

$$
\mathbf{H} = \begin{bmatrix}
K^2 & QK \\
QK & Q^2
\end{bmatrix}, \qquad
\operatorname{Tr}(\mathbf H) = Q^2 + K^2.
$$

On the contour $QK=C$, $\operatorname{Tr}(\mathbf H)$ is minimized when $|Q|=|K|=\sqrt{C}$, giving the flattest region.

b) One-step SGD stability
Update rules:
$$
Q_1 = Q_0 - \eta\, (Q_0 K_0 - C)\, K_0
$$
$$
K_1 = K_0 - \eta\, (Q_0 K_0 - C)\, Q_0
$$
Let $\mathbf g^{(1)}=[K_{1};\, Q_{1}]$ be the next gradient direction. A first-order expansion gives
$$
    \cos{\theta}=  1-\eta\, (Q_{0} K_{0}-C)(Q_{0}^{2}-K_{0}^{2})^{2}/||g||^{2}+O(\eta^{2})
$$
Hence $\cos{\theta}=  1$ if and only if $|Q_{0}|= |K_{0}|$, i.e., no directional oscillation.

c) Guaranteed sharpness reduction
After rescaling to $|Q'|= |K'|= \sqrt{C}$,
$$
    \Delta{\operatorname{Tr}}(\mathbf H)= Q_{0}^{2}+ K_{0}^{2}-  2C=(|Q_{0}|- |K_{0}|)^{2}\geqslant 0
$$
which guarantees a deterministic drop of curvature before any gradient step.

\begin{table*}[!h]
\caption{All metrics of table\ref{tab:llama_metrics}}
\label{tab:llama_metrics_all}
\begin{center}
\begin{small}
\begin{sc}
\setlength{\tabcolsep}{2pt}
\begin{tabular}{lccccccccc}
\toprule
 version & BoolQ↑ & ARC-c↑ & ARC-e↑ & PIQA↑ & WinoG↑ & HellaS↑ & OBQA↑ & ceval↑ & Avg↑ \\
\midrule
 origin & 0.3838 & 0.1741 & \textbf{0.3274} & 0.5288 & 0.5004 & 0.2650 & 0.1300 & 0.2437 & 0.3192 \\
 \ding{172} qk\_ten & 0.3810 & 0.1843 & 0.3228 & 0.5332 & 0.4807 & 0.2617 & 0.1280 & 0.2511 & 0.3179 \\
 \ding{173} qk\_row & 0.3887 & 0.1817 & 0.3102 & 0.5370 & 0.5091 & 0.2617 & 0.1280 & 0.2533 & 0.3212 \\
 \ding{174} vo\_ten & 0.4015 & 0.1852 & 0.3081 & 0.5294 & 0.4957 & 0.2653 & \textbf{0.1380} & 0.2585 & 0.3227 \\
 \ding{175} vo\_row & 0.3817 & 0.1749 & 0.3178 & 0.5386 & 0.4988 & 0.2629 & 0.1500 & 0.2548 & 0.3224 \\
\midrule
 \ding{172}+\ding{174} & \textbf{0.5214} & 0.1869 & 0.3165 & 0.5413 & 0.4980 & 0.2634 & 0.1340 & 0.2585 & \textbf{0.3400} \\
 \ding{173}+\ding{174} & 0.4483 & \textbf{0.2014} & 0.3165 & 0.5305 & 0.5059 & \textbf{0.2656} & 0.1320 & 0.2578 & 0.3323 \\
 \ding{172}+\ding{175} & 0.4226 & 0.1783 & 0.3199 & 0.5310 & 0.5075 & 0.2645 & 0.1260 & 0.2377 & 0.3234 \\
 \ding{173}+\ding{175} & 0.3994 & 0.1724 & 0.3224 & \textbf{0.5430} & \textbf{0.5170} & 0.2621 & 0.1340 & 0.2377 & 0.3235 \\
\midrule
 \ding{172}+\ding{174}(init) & 0.4703 & 0.1860 & 0.3182 & 0.5299 & 0.4964 & 0.2637 & 0.1300 & \textbf{0.2623} & 0.3321 \\
 \ding{173}+\ding{174}(init) & 0.3960 & 0.1792 & 0.3207 & 0.5386 & 0.5036 & 0.2652 & 0.1120 & 0.2355 & 0.3189 \\
 \ding{172}+\ding{175}(init) & 0.4312 & 0.1766 & 0.3203 & 0.5308 & 0.4949 & 0.2644 & 0.1220 & 0.2541 & 0.3243 \\
 \ding{173}+\ding{175}(init) & 0.3917 & 0.1817 & 0.3199 & 0.5337 & 0.5043 & 0.2631 & 0.1240 & 0.2311 & 0.3187 \\
\bottomrule
\end{tabular}
\end{sc}
\end{small}
\end{center}
\vskip -0.1in
\end{table*}

\begin{table*}[!h]
\caption{Zero-shot evaluation results on different training steps on llama\_moe-5B-A0.8B.}
\label{tab:zero-shot-results_all}
\begin{center}
\begin{small}
\begin{sc}
\begin{tabular}{l|ccccccc}
\toprule
Metric & 1w & 2w & 3w & 4w & 5w & 6w & 7w \\
\midrule
Trained Tokens & 1.43B & 2.86B & 4.29B & 5.72B & 7.15B & 8.58B & 10B \\
\midrule
\multicolumn{8}{l}{\textit{Baseline}} \\
\midrule
ARC-c & 19.28 & 18.69 & 17.66 & 19.11 & 20.22 & 19.20 & 20.73 \\
ARC-e & 32.74 & 37.54 & 41.33 & 42.80 & 43.43 & 43.56 & 44.57 \\
BoolQ & 48.10 & 60.06 & 60.55 & 62.05 & 61.13 & 61.07 & 49.36 \\
H-Swag & 26.80 & 27.32 & 27.86 & 28.47 & 28.97 & 29.44 & 29.84 \\
OpenBQ & 11.80 & 12.80 & 13.80 & 13.40 & 14.00 & 16.00 & 15.60 \\
PIQA & 57.73 & 58.60 & 60.39 & 61.53 & 62.40 & 63.11 & 62.89 \\
Wino & 52.49 & 50.12 & 53.35 & 51.46 & 50.75 & 49.80 & 52.72 \\
AVG & 35.56 & 37.88 & 39.28 & 39.83 & 40.13 & 40.31 & 39.39 \\
PPL(wikitext2) & 70.93 & 48.17 & 40.72 & 35.48 & 33.29 & 31.16 & 29.45 \\
\midrule
\multicolumn{8}{l}{\textit{QK}} \\
\midrule
ARC-c & 18.26 & 18.43 & 18.60 & 17.83 & 18.77 & 18.34 & 20.22 \\
ARC-e & 33.00 & 37.58 & 39.86 & 42.68 & 44.19 & 45.29 & 44.78 \\
BoolQ & 49.11 & 62.23 & 62.08 & 62.17 & 61.87 & 61.80 & 61.41 \\
H-Swag & 26.65 & 27.34 & 28.17 & 28.57 & 28.95 & 29.43 & 29.74 \\
OpenBQ & 12.00 & 14.00 & 14.60 & 15.00 & 16.00 & 15.00 & 14.60 \\
PIQA & 57.13 & 59.63 & 60.88 & 62.35 & 63.00 & 63.66 & 64.31 \\
Wino & 50.67 & 52.41 & 50.75 & 51.85 & 50.51 & 52.88 & 52.49 \\
AVG & 35.26 & 38.80 & 39.28 & 40.06 & 40.47 & 40.91 & 41.08 \\
Metric.Gain(\%) & -0.84 & \textbf{+2.43} & 0.00 & \textbf{+0.58} & \textbf{+0.85} & \textbf{+1.49} & \textbf{+4.29} \\
PPL(wikitext2) & 69.44 & 47.12 & 39.74 & 34.97 & 32.86 & 30.99 & 29.19 \\
PPL.Gain(\%) & \textbf{+2.10} & \textbf{+2.19} & \textbf{+2.42} & \textbf{+1.43} & \textbf{+1.29} & \textbf{+0.54} & \textbf{+0.85} \\
\midrule
\multicolumn{8}{l}{\textit{QKVO}} \\
\midrule
ARC-c & 17.58 & 18.60 & 18.09 & 19.97 & 18.77 & 20.39 & 18.77 \\
ARC-e & 35.40 & 38.22 & 40.74 & 43.77 & 43.94 & 44.82 & 44.19 \\
BoolQ & 60.70 & 61.80 & 60.09 & 61.74 & 61.68 & 61.80 & 59.79 \\
H-Swag & 26.75 & 27.49 & 28.12 & 28.69 & 28.89 & 29.40 & 29.89 \\
OpenBQ & 11.40 & 13.00 & 16.40 & 15.60 & 16.40 & 15.40 & 17.80 \\
PIQA & 58.38 & 59.41 & 61.26 & 62.68 & 62.62 & 63.11 & 63.60 \\
Wino & 49.25 & 49.09 & 51.70 & 52.33 & 50.04 & 49.88 & 50.99 \\
AVG & 37.07 & 38.23 & 39.49 & 40.68 & 40.33 & 40.69 & 40.72 \\
Metric.Gain(\%) & \textbf{+4.25} & \textbf{+0.92} & \textbf{+0.53} & \textbf{+2.13} & \textbf{+0.50} & \textbf{+0.94} & \textbf{+3.38} \\
PPL(wikitext2) & 68.11 & 46.25 & 39.17 & 34.85 & 32.77 & 30.80 & 29.21 \\
PPL.Gain(\%) & \textbf{+3.98} & \textbf{+3.98} & \textbf{+3.82} & \textbf{+1.77} & \textbf{+1.56} & \textbf{+1.15} & \textbf{+0.82} \\
\bottomrule
\end{tabular}
\end{sc}
\end{small}
\end{center}
\vskip -0.1in
\end{table*}

\end{document}